\documentclass[preprint,11pt, a4paper]{elsarticle}
\usepackage{amssymb,amsthm}
\usepackage{amsmath}

\numberwithin{equation}{section} 
\usepackage[ruled,vlined]{algorithm2e}
\usepackage{url}
\usepackage{caption}
\usepackage{subcaption}
\usepackage{xcolor}
\usepackage{array}
\usepackage[colorlinks=true,linkcolor=green,citecolor=green]{hyperref}
\usepackage[a4paper,left=3cm,right=3cm,top=2.5cm,bottom=2.5cm]{geometry}
\newtheorem{theorem}{Theorem}

\begin{document}
	
	\begin{frontmatter}
		
		
		\title{Improving population size adapting CMA-ES algorithm on step-size blow-up in weakly-structured multimodal functions}
		
		\author{Chandula Fernando\corref{cor1}\fnref{label1}}
		\ead{ chandufernando99@gmail.com}
		\cortext[cor1]{Corresponding author.}
		
		\author{Kushani De Silva\fnref{label2}}
		\ead{ kpdesilva@uncg.edu}
		
		

		\begin{abstract}
			
			Multimodal optimization requires both exploration and exploitation. Exploration identifies promising attraction basins, while exploitation finds the best solutions within these basins. The balance between exploration and exploitation can be maintained by adjusting parameter settings. The population size adaptation covariance matrix adaption evolutionary strategy algorithm (PSA-CMA-ES) achieves this balance by dynamically adjusting population size. PSA-CMA-ES performs well on well-structured multimodal benchmark problems. In weakly structured multimodal problems, however, the algorithm struggles to effectively manage step-size increases, resulting in uncontrolled step-size blow-ups that impede convergence near the global optimum. In this study, we reformulated the step-size correction strategy to overcome this limitation. We analytically identified the cause of the step-size blow-up and demonstrate the existence of a significance level for population size change guiding a safe passage to step-size correction. These insights were incorporated to form the reformulation. Through computer experiments on two weakly structured multimodal benchmark problems, we evaluated the performance of the new approach and compared the results with the state-of-the-art algorithm. The improved algorithm successfully mitigates step-size blow-up, enabling a better balance between exploration and exploitation near the global optimum enhancing convergence.
			
		\end{abstract}

		\begin{keyword}
			evolutionary optimization \sep step-size blow-up \sep adaptation \sep covariance matrix \sep multimodal \sep weakly-structured
			
		\end{keyword}
		
	\end{frontmatter}

	\section{Introduction}
	\label{Sec:Intro}
	
	Inspired by Darwinian principles of natural selection, evolutionary algorithms solve complex optimization problems by selecting the best solutions at each generation to reproduce. This evolutionary process iteratively explores and refines solutions at promising basins of the search space. Exploration of the search space requires covering a larger area, while exploiting promising regions concentrates on a smaller area balancing exploration and exploitation to locate the global optimum. Maintaining this balance is critical in evolutionary optimization algorithms \citep{ES_11}. The Covariance Matrix Adaptation Evolutionary Strategy algorithm (CMA-ES) is a state-of-the-art stochastic evolutionary algorithm designed to solve complex non-linear, non-convex, continuous black-box optimization problems \cite{CMA_2,CMA_5,CMA_12,CMA_17}. Developed by N. Hansen and A. Ostermeier in 1996, the algorithm adapts the covariance matrix over iterative generations guiding the search distribution to converge at the optimum \cite{CMA_1}. Generally, CMA-ES is applied when derivative-based methods fail and has shown competitive behavior in non-convex, non-smooth, non-separable, as well as multimodal functions \cite{CMA_3,CMA_5}. A key feature of the CMA-ES algorithm is its quasi-parameter-free nature where all parameters (e.g. learning rates, recombination weights) depend only on the dimension of the problem. This makes CMA-ES popular among practitioners as it alleviates the expensive trial-and-error approach of parameter tuning and has been used in various real-world applications \cite{Appli_1,Appli_2,Appli_3,Appli_4,Appli_5,Appli_6}. However, in CMA-ES the population size of the candidate solutions (\(\lambda\)) sometimes requires tuning. A population size larger than the default value of \( \lambda_{\text{def}} = 4 + \lfloor3 \ln(n)\rfloor\) has been shown to improve the quality of solutions obtained by CMA-ES when applied to multimodal functions \cite{CMA_4,ES_9,CMA_11}. Here, \(\lfloor . \rfloor\) represents the floor function that rounds a number down to the nearest integer less than or equal to its original value. For example, in multimodal problems, a large population helps effectively explore the search space, but a  smaller population size is sufficient once the algorithm converges to a promising region. Thus a satisfactory value for \(\lambda\) may vary throughout the optimization process \cite{PSA_1}. \\
	
	Thus in 2018, a population size adapting variant of CMA-ES, called Population Size Adaptation Covariance Matrix Adaptation Evolutionary Strategy (PSA-CMA-ES) was introduced. This algorithm adapts the population size by estimating the update accuracy of all distribution parameters such as \(\mathbf{m}\), \(\sigma\), and \(\mathbf{C}\). The goal is to quantify the accuracy of the natural gradient estimation at the current population size at the current situation. To achieve this, they introduced an evolution path defined in the distribution parameter space, and the population size adapted by taking into account cues from this evolution path. The relationship between step-size and increasing population size, as studied in \cite{ES_10}, necessitated a step-size correction to account for changes in step-size caused by increases in population size \cite{PSA_2}. Thus, PSA-CMA-ES does a step-size correction following the population size adaptation. Combined with a restart strategy, PSA-CMA-ES works well on well-structured multimodal problems. However, it fails to deliver similar results on weakly-structured multimodal problems, e.g. Rastrigin, Schaffer functions \cite{PSA_3}. In particular, the algorithm's tendency to continually increase the step-size over generations prevents convergence, wasting valuable computational resources. Further, PSA-CMA-ES fails to adapt the population size across generations leading to poor performance on weakly-structured two-dimensional multimodal functions. Additionally, since PSA-CMA-ES keeps increasing step-size, after about 10-15 generations, the algorithm gets stuck in a loop, significantly increasing CPU time without improving results. These shortcomings which became evident when tested on weakly-structured multimodal two-dimensional Rastrigin and Schaffer functions, underscore the need for a more efficient approach. \\

    This paper introduces a novel reformulation of the step-size correction mechanism in PSA-CMA-ES, specifically designed to address challenges with exploration and convergence in weakly-structured functions. The proposed approach effectively avoids step-size blow-up near optima ensuring a more precise convergence. We analytically demonstrate that the continual blow up of step-size observed in PSA-CMA-ES is a direct consequence of the step-size correction mechanism. We further show the existence of a significance level for population size change determining when the step-size correction becomes necessary. Accounting for these factors the reformulated step-size correction mechanism leverages the population size adaptation to improve exploration and exploitation of the search space by avoiding step-size blow-up. Moreover, we experimentally highlighted the importance of retaining a controlled form of the step-size correction since removing it entirely leads to premature convergence. Our experiments tested on Rastrigin and Schaffer, two dimensional weakly-structured multimodal benchmark problems, confirming the theoretical insights and showed higher performance with respect to the state-of-the-art PSA-CMA-ES algorithm in terms of computational efficiency and convergence. \\

	The rest of this paper is structured as follows. Section \ref{sec:PSA-CMA-ES} provides an overview of the CMA-ES algorithm and its population-size adapting variant, PSA-CMA-ES. Section \ref{sec:New-Reformulation} presents the reformulation, organized into subsections. Section \ref{Analysis} provides an analysis of the reformulated algorithm. Section \ref{sec:Exp-Setting} offers experimental evidence supporting the analysis of step-size blow-up in the general PSA-CMA-ES. Section \ref{sec:Reformulation} introduces the reformulated algorithm. Section \ref{sec:Implementation_Results} reports the performance results of the reformulation compared to the general PSA-CMA-ES. Finally, Section \ref{sec:Conclusion} summarizes the findings of this study and discusses potential directions for future research.

	\section{PSA-CMA-ES}
	\label{sec:PSA-CMA-ES}
	
		\subsection{General CMA-ES}
		\label{subsec:CMA-ES}
		
	The CMA-ES algorithm, minimizing the \(n\)-dimensional function  \( f: \mathbb{R}^{n} \rightarrow \mathbb{R} \) samples its candidate solutions from a multivariate normal distribution \( N(\mathbf{m}, \sigma^2 \mathbf{C})\) where \(\mathbf{m}\), \(\sigma\), and \(\mathbf{C}\) respectively represent the mean vector, step-size, and covariance matrix \cite{CMA_1}. The initial mean vector \(\mathbf{m}^{(0)}\) and covariance matrix \(\mathbf{C}^{(0)} = \mathbf{I}\) are predefined and the step-size \(\sigma^{(0)}\) is initialized to half of function’s initialization interval.  At each generation a $\lambda_\text{def}$ sized population of candidate solutions are sampled from this multivariate normal distribution \cite{ES_4},
	\begin{equation}\label{Eq: lam_default}
	\lambda_\text{def} = 4+\lfloor  3\ln(n) \rfloor.
	\end{equation}
	 These candidates are evaluated on the objective function \(f\), and ranked based on their performance. In general, the \(\lfloor\frac{\lambda}{2}\rfloor\) candidates demonstrating best fitness are selected as ``parents'' of the next generation and their mean vector is shifted by a weighted rank index,	
	\begin{equation} \label{eq:2}
		\textbf{m}^{(g+1)} = \textbf{m}^{(g)} + c_m\sum_{i=1}^{\mu}w_i(\textbf{x}^{(g+1)}_{i:\lambda} - \textbf{m}^{(g)}).
	\end{equation}
	Here, \(\textbf{x}_{i:\lambda}\) refers to the \(i^{\text{th}}\) best candidate, \(w_{i}\) represents the assigned weight and \(c_{m}\) is the learning rate of the mean vector. This updates the mean vector \(\textbf{m}^{(g+1)}\) as a weighted average of the best-performing candidates from the previous generation to shift the algorithm towards regions of higher fitness. \\

	\subsubsection{Step-Size Adaptation of CMA-ES}
	\label{sec:New Reformulation}
	
	In the CMA-ES algorithm, the step-size or \(\sigma\) adaptation is independent of the covariance matrix update. The adaptation mechanism, known as the cumulative step-size adaptation (CSA), accumulates the steps of mean vector updates using a conjugate evolution path \(\textbf{p}_{\sigma}\) where \cite{ES_4}
	\begin{equation} \label{eq:4}
		\textbf{p}^{(g+1)}_{\sigma} = (1-c_{\sigma}) \textbf{p}^{(g)}_{\sigma} + \sqrt{c_{\sigma}(2-c_{\sigma})\mu_{\text{eff}}} 
		\left(\textbf{C}^{(g)}\right)^{-\frac{1}{2}} \left(\frac{\textbf{m}^{(g+1)} - \textbf{m}^{(g)}}{\sigma^{(g)}}\right).
	\end{equation}
	It is important to note here that multiplying by \(\textbf{C}^{(g)^{-\frac{1}{2}}}\) verifies that the expected length of \(\textbf{p}_{\sigma}\) becomes independent of its direction and \(\textbf{p}_{\sigma} \sim N(\textbf{0},\textbf{I})\). \(\mu_{\text{eff}}\) sometimes referred to as \(\mu_{w}\) and it is the variance-effective selection mass which determines the influence of selected candidates on \(\sigma\). If candidates are given equal weights (\(w_{i} = \frac{1}{\mu}\)), then \(\mu_{\text{eff}} = \mu\) \cite{CMA_3}. The step-size \(\sigma^{(g)}\) is updated by comparing the length of evolution path \(||\textbf{p}_{\sigma}||\) to its expected length, \( ||E[N(\textbf{0},\textbf{I})]||\), as shown in Eq. (\ref{eq:5}) below \cite{ES_4}.
	\begin{equation} \label{eq:5}
		\sigma^{(g+1)} = \sigma^{(g)} \text{exp}\left(\frac{c_{\sigma}}{d_{\sigma}} \left(\frac{||\textbf{p}^{(g+1)}_{\sigma}||}{||E[N(\textbf{0},\textbf{I})]||} - \sqrt{\gamma^{(g+1)}_{\sigma}} \right)\right).
	\end{equation}
	
	The damping paramter \(d_{\sigma} \geq 1\) limits undesirable variations of \(\sigma\) between generations \cite{ES_4}. Thus, \(\sigma\) is updated depending on how consecutive steps are correlated and these correlations are identified under three cases; (a) If steps are anti-correlated (length of evolution path $|| \textbf{p}_{\sigma}||$ is shorter than expected), they effectively cancel each other out. Consequently, the step-size should be reduced allowing exploitation, (b) If steps are correlated (length of evolution path $|| \textbf{p}_{\sigma}||$ is longer than expected), they point in the same direction. Consequently the step-size should be increased allowing more exploration. (c) If steps are uncorrelated (length of evolution path $|| \textbf{p}_{\sigma}||$ is $||E[N(0,\textbf{I})]||$), they are perpendicular. Consequently step-size is not updated.

	\subsubsection{Covariance matrix adaptation in CMA-ES}
	
	The covariance matrix adaptation (CMA) in CMA-ES involves two steps: the rank\(-1\) update and the rank\(-\mu\) update. Rank\(-1\) employs an evolutionary path \(\textbf{p}^{(g+1)}_{c}\) to update the covariance matrix where \(\textbf{p}_{c}\) evolution path accumulates successive steps as an exponentially smoothed sum following, 
	\begin{equation} \label{eq:pc}
		\textbf{p}^{(g+1)}_{c} = (1-c_{c}) \textbf{p}^{(g)}_{c} + h_{\sigma}\sqrt{c_{c}(2-c_{c})\mu_{\text{eff}}} \left( \frac{\textbf{m}^{(g+1)} - \textbf{m}^{(g)}}{\sigma^{(g)}}\right),
	\end{equation}
	where 
	\begin{equation}
h^{(g+1)}_{\sigma} = \begin{cases}
	1, \text{ if } ||\textbf{p}^{(g+1)}_{\sigma}||<\left( 1.4 + \frac{2}{n+1}\right)  E\left[ ||N(\textbf{0},\textbf{I})||\right]  \sqrt{\gamma^{(g+1)}_{\sigma}}\\
	0, \text{ otherwise}.
\end{cases}
	\end{equation}
	The next step, rank-\(\mu\), estimates the covariance matrix of the next generation by using a weighted recombination to select a potentially better covariance matrix by reproducing successful steps, and a maximum likelihood estimation, \((\textbf{x}^{(g+1)}_{i:\lambda} - \textbf{m}^{(g)})(\textbf{x}^{(g+1)}_{i:\lambda} - \textbf{m}^{(g)})^{T}\) that increases the variance in the direction of the natural gradient. The overall covariance update combines rank\(-1\) and rank\(-\mu\) \cite{CMA_17},
	\begin{align} \label{eq:cm_update}
			\textbf{C}^{(g+1)} &= \textbf{C}^{(g)} + c_{1}\left(\textbf{p}^{(g+1)}_{c}\left( \textbf{p}^{(g+1)}_{c}\right) ^{T} - \gamma^{(g+1)}_{c} \textbf{C}^{(g)}\right) \\   &+c_{\mu}\sum_{i=1}^{\lambda}w_i\left(\left( \textbf{x}^{(g+1)}_{i:\lambda} - \textbf{m}^{(g)}\right) \left( \textbf{x}^{(g+1)}_{i:\lambda} - \textbf{m}^{(g)}\right) ^{T} - \textbf{C}^{(g)}\right). \notag
	\end{align}
	The Heaviside step function, \(h_{\sigma}^{(g+1)}\) in Eq. (\ref{eq:pc}) determines the evolution path contribution to the update and  \(E[||N(\textbf{0},\textbf{I})||] = \sqrt{2}\frac{\Gamma(\frac{n+1}{2})}{\Gamma(\frac{n}{2})} \approx \sqrt{n}(1 - \frac{1}{(4n)} + \frac{1}{21n^{2}})\) is an approximation for the expected norm of the n-variate standard normal distribution \cite{CMA_5}. Two normalization factors, \(\gamma^{(g+1)}_{\sigma}\) and \(\gamma^{(g+1)}_{c}\), which converge to \(1\) as \(g\) increases, are introduced for a clean derivation of the population size adaptation mechanism. Although their effects are not recognizable significantly in implementations, they are included for a well-rounded perspective \cite{PSA_2}.

	\subsection{State-of-the-art PSA-CMA-ES} \label{generalPSA}
	
    This section explains the population size adaptation mechanism applied to CMA-ES explained in Section \ref{subsec:CMA-ES}. The PSA-CMA-ES algorithm extends the general CMA-ES algorithm by adapting the population size \(\lambda\) at every generation \cite{PSA_2}. In contrast to CMA-ES, PSA-CMA-ES update the population size $\lambda$ at every generation.\\

   The population size $\lambda$ is adapted at every generation based on the length of evolution path \(\textbf{p}_{\theta}\) \cite{PSA_3,PSA_2},
    \begin{equation} \label{eq:pop-size_update}
    	\lambda^{(g+1)} \leftarrow \lambda^{(g)} \text{exp}\left[\beta\left((\gamma^{(g+1)}_{\theta}) - \frac{||\textbf{p}^{(g+1)}_{\theta}||^{2}}{\alpha}\right)\right], 
    \end{equation}
   with $\lambda^{(0)}$ is the default population size in CMA-ES \cite{CMA_17}. The parameter $\gamma_{\theta}^{(g+1)}$ is referred to as a normalization factor for $\textbf{p}_{\theta}$ defined by $(1-\beta)^2\gamma_{\theta}^{(g)} + \beta (2-\beta)$ which converges to 1 as $g$ increases. Additionally, $\beta=0.4$ is the population size learning parameter, $\gamma_{\theta}^{(0)}=0$, and $\alpha=1.4$ \cite{PSA_2}. The evolution path $\textbf{p}_{\theta}$ is given by, 
   	\begin{equation} \label{eq:p_theta}
   	\textbf{p}^{(g+1)}_{\theta} = (1-\beta)\textbf{p}^{(g)}_{\theta} + \sqrt{\beta(2-\beta)} \frac{\sqrt{\textbf{F}}^{(g)}\Delta\theta^{(g+1)}}{\text{E}\left[ \left\lVert\sqrt{\textbf{F}}^{(g)}\Delta\theta^{(g+1)}\right\rVert^{2}\right] ^{\frac{1}{2}}},
   \end{equation}
where \textbf{F} is the Fisher Information matrix, $\text{E}$ denotes the expected value, and $||. ||$ denotes the $L-2$ norm. The matrix $\textbf{F}$ can be approaximated by $\textbf{C}^{-1}$ since candidate solutions are normally distributed. Further, in \cite{PSA_3,PSA_2}, the population size update is forced within two bounds such that,
    \begin{equation} \label{eq:pop-size_boundary}
	\lambda^{(g+1)}_{r} = \begin{cases}
		\lambda_{\text{min}}, \quad &\text{if} \quad  \lambda^{(g+1)} \leq \lambda_{\text{min}},\\
		\text{round }(\lambda^{(g+1)}), \quad &\text{if} \quad  \lambda_{\text{min}} < \lambda^{(g+1)} < \lambda_{\text{max}},\\
		\lambda_{\text{max}},\quad &\text{if} \quad  \lambda^{(g+1)} \geq \lambda_{\text{max}},
	\end{cases}
\end{equation}
  where $\lambda_{\text{min}} = \lambda$ \cite{CMA_17} and $\lambda_{\text{max}}=512\, \lambda$. The vector $\Delta \theta^{(g+1)} \in \mathbb{R}^{(n(n+3)/2)}$ in Eq.\eqref{eq:p_theta} records the parameter evolution after each generation $g$,
   	\begin{equation}  \label{eq:paramter_update_vector}
{\Delta\theta}^{(g+1)} = \left( \Delta \textbf{m}^{(g+1)}, \text{vech}\left( \Delta\Sigma^{(g+1)}\right) \right),
   \end{equation} 
   	where $\Delta \textbf{m}^{(g+1)} = \textbf{m}^{(g+1)} - \textbf{m}^{(g)}$ and
   $\Delta\Sigma^{(g+1)} = \left( \sigma^{(g+1)}\right) ^{2}\textbf{C}^{(g+1)} - \left( \sigma^{(g)}\right) ^{2}\textbf{C}^{(g)}$ with $\Sigma = \sigma^2\textbf{C}$. 
   The vector $\text{vech}(\Sigma)$ arranges the $(i,j)^{\text{th}}$ entry of $\Sigma$ into the $\text{vech}(\Sigma)$ at entry $\left(i-j+1+\sum_{k=1}^{j-1}(n-k+1)\right)$.\\

   However, adapting population size causes undue blow-ups in the step-size $\sigma$ leading to unstable adaptations \cite{ES_7}. To address this fault, a step-size correction was introduced in \cite{PSA_2}. The $\sigma$ correction mechanism is given below.
	\begin{equation} \label{eq:step-size_correction}
		\sigma_c^{(g+1)} = \sigma^{(g+1)} \frac{\rho\left( \lambda^{(g+1)}_{r}\right) }{\rho\left( \lambda^{(g)}_{r}\right) },
	\end{equation}	
where the scaling factor $\rho(\cdot)$ is given by,
	\begin{equation} \label{eq:step-size-correction_rho_eqn}
		\rho\left( \lambda^{(g)}_{r}\right)  = \frac{ n \left(-\sum_{i=1}^{\lambda_{r}}w_i \text{E}[\mathcal{N}_{i:\lambda_{r}}]\right) \mu_{\text{w}}}{\left[n-1+\left(\left(-\sum_{i=1}^{\lambda_{r}}w_i \text{E}[\mathcal{N}_{i:\lambda_{r}}]\right)^{2}\mu_{\text{w}}\right)\right]},
	\end{equation}	
	where $\text{E}[\mathcal{N}_{i:\lambda_{r}}] = \textbf{m} + \sigma \, \text{E}\left[ N(\textbf{0},\textbf{I})\right] $ is the expected normal order statistic for sorted candidate solutions with respect to performance, i.e. value of $f(\textbf{x}_i)$ \cite{PSA_2}. However an approximation to this expected normal order statistic is given in \cite{NOS_1},
		\begin{equation} \label{eq:Expected_NOS}
		\text{E}[\mathcal{N}_i]  \approx \mu + \phi^{-1} \left(\frac{i- \alpha_{1}}{\lambda_r - 2\alpha_{1} + 1}\right), \quad \text{ for } i=1,\cdots,\lambda_r,
	\end{equation}
	where $\phi$ is the cumulative normal distribution function and  $\alpha_{1}=0.375$. Although this step-size correction mechanism was introduced in \cite{PSA_2}, its performance on benchmark problems was evaluated in \cite{PSA_3}. The results in \cite{PSA_3} still exhibited a step-size blow-up, even with the mechanism from \cite{PSA_2} applied. To address this issue, \cite{PSA_3} proposed a new restart mechanism where the algorithm uses CMA-ES optimum value upon its crash as the initial start of the PSA-CMA-ES. Then another restart uses PSA-CMA-ES with a small initial step-size $\sigma^{(0)}\sim 2\times10^{-2\mathrm{Uni}[0,1]}$.

\section{Reformulation to PSA-CMA-ES}
	\label{sec:New-Reformulation}
This section presents the results of our analysis on the reformulation of the PSA-CMA-ES algorithm, along with the corresponding numerical evidence supporting the findings. The final part of the section compares the performance of the reformulated algorithm with the general state-of-the-art PSA-CMA-ES algorithm, using evaluations on two-dimensional Rastrigin and Schaffer benchmark functions.

\subsection{Mathematical analysis}
         \label{Analysis}

\subsubsection{Analysis of step-size ($\sigma$) blow-up }
	\label{step-size blow-up}
	
As discussed in the previous section, the step-size correction has used the normal order statistic in its mechanism (see Eqs. \eqref	{eq:step-size_correction},\eqref{eq:step-size-correction_rho_eqn},\eqref{eq:Expected_NOS}). When PSA-CMA-ES is applied to a minimization problem, the expected mean reduces as the generation progresses, i.e.  \(\textbf{m}^{(g+1)} \leq \textbf{m}^{(g)}\), as \(g \rightarrow g+1\). This results in reducing the approximated expected normal order statistic as generation progresses, i.e. $\text{E}[\mathcal{N}_{i:\lambda_{r}^{(g+1)}}] \leq \text{E}[\mathcal{N}_{i:\lambda_{r}^{(g)}}]$. Thus in the following theorem, we demonstrate why the blow-up in step-size is takes place in PSA-CMA-ES.\\

	\begin{theorem}\label{Theorem:Factor>1}
 If $\text{E}[\mathcal{N}_{i:\lambda_{r}^{(g+1)}}] \leq \text{E}[\mathcal{N}_{i:\lambda_{r}^{(g)}}]$ and $ \textbf{m}^{(g+1)} \leq \textbf{m}^{(g)}$ with $\mu_{w}^{(g+1)} - \mu_{w}^{(g)} \leq \delta$ where $\delta$ is a small positive value, then ${\sigma_c^{(g+1)}}/{\sigma^{(g+1)}} \geq 1$.
	\end{theorem}

	\begin{proof}
It is given in PSA-CMA-ES that,
	\begin{equation} \label{eq:line2}
	\text{E}[\mathcal{N}_{i:\lambda_{r}^{(g+1)}}] \quad \leq \quad \text{E}[\mathcal{N}_{i:\lambda_{r}^{(g)}}]. \notag
\end{equation}
It yields,
	\begin{equation} \label{eq:line3} 
	\sum_{i=1}^{\lambda_{r}}w_i \text{E}[\mathcal{N}_{i:\lambda_{r}^{(g+1)}}] \quad \leq \quad \sum_{i=1}^{\lambda_{r}}w_i \text{E}[\mathcal{N}_{i:\lambda_{r}^{(g)}}],
\end{equation}
and,
	\begin{equation} \label{eq:line5} 
	\left(-\sum_{i=1}^{\lambda_{r}}w_i \text{E}[\mathcal{N}_{i:\lambda_{r}^{(g+1)}}]\right)n\mu_{\text{w}}^{(g+1)} \quad \geq \quad \left(-\sum_{i=1}^{\lambda_{r}}w_i \text{E}[\mathcal{N}_{i:\lambda_{r}^{(g)}}]\right)n\mu_{\text{w}}^{(g)}.
\end{equation}
By \eqref{eq:line3} we have,
	\begin{equation} \label{eq:line} 
	\left( -\sum_{i=1}^{\lambda_{r}}w_i \text{E}[\mathcal{N}_{i:\lambda_{r}^{(g+1)}}]\right) ^2 \quad \leq \quad \left( -\sum_{i=1}^{\lambda_{r}}w_i \text{E}[\mathcal{N}_{i:\lambda_{r}^{(g)}}]\right) ^2. \notag
\end{equation}
When $\lim_{\delta \to 0}$, we have $\mu_{w}^{(g+1)} < \mu_{w}^{(g)} $ which yields,
	\begin{equation} \label{eq:line_} 
	\left( -\sum_{i=1}^{\lambda_{r}}w_i \text{E}[\mathcal{N}_{i:\lambda_{r}^{(g+1)}}]\right) ^2 \mu_{w}^{(g+1)} \quad \leq \quad \left( -\sum_{i=1}^{\lambda_{r}}w_i \text{E}[\mathcal{N}_{i:\lambda_{r}^{(g)}}]\right) ^2 \mu_{w}^{(g)}. \notag
\end{equation}
Since $(n-1)\geq 0$, we can write
	\begin{equation}\label{eqn: 15}
(n-1) +	\left( -\sum_{i=1}^{\lambda_{r}}w_i \text{E}[\mathcal{N}_{i:\lambda_{r}^{(g+1)}}]\right) ^2 \mu_{w}^{(g+1)} \quad \leq \quad (n-1) + \left( -\sum_{i=1}^{\lambda_{r}}w_i \text{E}[\mathcal{N}_{i:\lambda_{r}^{(g)}}]\right) ^2 \mu_{w}^{(g)}.
\end{equation}
From Eqs. \eqref{eq:line5} and \eqref{eqn: 15} we can derive,
\begin{equation}\label{eq 16}
\dfrac{	\left(-\sum_{i=1}^{\lambda_{r}}w_i \text{E}[\mathcal{N}_{i:\lambda_{r}^{(g+1)}}]\right)n\mu_{\text{w}}^{(g+1)} }{(n-1) +	\left( -\sum_{i=1}^{\lambda_{r}}w_i \text{E}[\mathcal{N}_{i:\lambda_{r}^{(g+1)}}]\right) ^2 \mu_{w}^{(g+1)}} \geq \dfrac{ \left(-\sum_{i=1}^{\lambda_{r}}w_i \text{E}[\mathcal{N}_{i:\lambda_{r}^{(g)}}]\right)n\mu_{\text{w}}^{(g)}}{(n-1) +	\left( -\sum_{i=1}^{\lambda_{r}}w_i \text{E}[\mathcal{N}_{i:\lambda_{r}^{(g+1)}}]\right) ^2 \mu_{w}^{(g+1)}},
\end{equation}
and
\begin{equation}\label{eq 17}
	\dfrac{ \left(-\sum_{i=1}^{\lambda_{r}}w_i \text{E}[\mathcal{N}_{i:\lambda_{r}^{(g)}}]\right)n\mu_{\text{w}}^{(g)}}{(n-1) +	\left( -\sum_{i=1}^{\lambda_{r}}w_i \text{E}[\mathcal{N}_{i:\lambda_{r}^{(g+1)}}]\right) ^2 \mu_{w}^{(g+1)}} \geq \dfrac{ \left(-\sum_{i=1}^{\lambda_{r}}w_i \text{E}[\mathcal{N}_{i:\lambda_{r}^{(g)}}]\right)n\mu_{\text{w}}^{(g)}}{(n-1) + \left( -\sum_{i=1}^{\lambda_{r}}w_i \text{E}[\mathcal{N}_{i:\lambda_{r}^{(g)}}]\right) ^2 \mu_{w}^{(g)}}.
\end{equation}
Then from Eqs. \eqref{eq 16} and \eqref{eq 17} we obtain,
\begin{equation}
\dfrac{	\left(-\sum_{i=1}^{\lambda_{r}}w_i \text{E}[\mathcal{N}_{i:\lambda_{r}^{(g+1)}}]\right)n\mu_{\text{w}}^{(g+1)} }{(n-1) +	\left( -\sum_{i=1}^{\lambda_{r}}w_i \text{E}[\mathcal{N}_{i:\lambda_{r}^{(g+1)}}]\right) ^2 \mu_{w}^{(g+1)}} \geq  \dfrac{ \left(-\sum_{i=1}^{\lambda_{r}}w_i \text{E}[\mathcal{N}_{i:\lambda_{r}^{(g)}}]\right)n\mu_{\text{w}}^{(g)}}{(n-1) + \left( -\sum_{i=1}^{\lambda_{r}}w_i \text{E}[\mathcal{N}_{i:\lambda_{r}^{(g)}}]\right) ^2 \mu_{w}^{(g)}},
\end{equation}
	i.e.,
\begin{equation} \label{eq:line9} 
	\rho\left(\lambda_{r}^{(g+1)}\right) \quad \geq \quad \rho\left(\lambda_{r}^{(g)}\right), \notag
\end{equation}
where $\rho$ is given in Eq. \eqref{eq:step-size_correction} concluding
\begin{equation} \label{eq:line10} 
	\frac{\rho\left(\lambda_{r}^{(g+1)}\right)}{\rho\left(\lambda_{r}^{(g)}\right)} \geq 1, \notag
\end{equation}
and with Eq. \eqref{eq:step-size_correction},
\begin{equation}
	\dfrac{\sigma_c^{(g+1)}}{\sigma^{(g+1)}} \geq 1, \notag
\end{equation}
i.e. the step-size correction ($\sigma_c$) increases.
	\end{proof}

According to the results of Theorem 1, the step-size progressively increases across successive generations, as each generation inherits its initial step-size from the amplified step-size of the previous generation, thereby hindering convergence to the global optimum. This increase in step-size takes place irrespective of the step-size adaptation in Eq.~\eqref{eq:5}. On the other hand, as stated in Eq. \eqref{eq:pop-size_boundary}, if $\lambda_r$ falls below the defined lower bound, it will remain constant across generations. Under this condition, the population size ensures that $\left( \mu_w^{(g+1)}-\mu_w^{(g)} \right) \to 0$, ultimately leading to a step-size blow-up. Since $\mu_w = g(\lambda)$ in general PSA-CMA-ES (see Algorithm \ref{alg:CMA-ES} lines 4-5), this leads to observing the blow-up more prominently when the difference in population size is insignificant, i.e. $\lambda^{(g+1)} - \lambda^{(g)}<L$ for some $L$ where $\delta=qL, \, q \in \mathbb{R}$. In the next section this threshold which makes population size change insignificant is discussed. 

\subsubsection{Analysis of population size change ($\Delta \lambda$)}

 In Theorem \ref{Theorem:Factor>1}, we showed that the step-size blows up when $\left( \mu_w^{(g+1)}-\mu_w^{(g)} \right) \to 0$. As discussed in the previous section, variance effective selection mass $\mu_w$ is a function of population size $\lambda$ in general PSA-CMA-ES (see Algorithm \ref{alg:CMA-ES} lines 4-5). Therefore, the step-size blow-up is prominent when the change in population size is insignificant, i.e.  $\lambda^{(g+1)} - \lambda^{(g)}<L$ for some $L$ where $\delta=qL, \, q \in \mathbb{R}$. In this section we determine the function $g(\cdot)$ by linking $\mu_w$ and $\lambda$ to derive the corresponding threshold $L$.

\begin{theorem}\label{Theorem:Delta_lambda} 
There exists a threshold $L \in \mathbb{N}$ for
\begin{equation} 
		\lambda^{(g+1)} - \lambda^{(g)} \leq L,
		 \end{equation} 
for some $L$ such that  $\delta = qL, \,q\in \mathbb{R}$ satisfying $\mu_{w}^{(g+1)} - \mu_{w}^{(g)} \leq \delta$.
\end{theorem}

\begin{proof} 
	
	From Theorem \ref{Theorem:Factor>1}, step-size blows up when $\mu_{w}^{(g+1)} - \mu_{w}^{(g)} \leq \delta$. Further, according to Algorithm 2, $\mu_w = g\left(\lambda\right) $. Due to the complexity of its analytical form in the algorithm,  $g\left( \cdot\right) $ was  numerically estimated using curve fitting. Accordingly, the best fit estimation yields the following linear function,
		\begin{equation}
		\mu_w = 0.2642 \, \lambda + 0.5328. \notag
	\end{equation}
		Substituting this result into $\mu_{w}^{(g+1)} - \mu_{w}^{(g)} \leq \delta$ yields,
	\begin{equation}
		\label{eq:proof2_sub2}
		0.2642 \, \left( \lambda^{(g+1)} - \lambda^{(g)}\right)  \leq \delta. \notag
	\end{equation}
	Simplifying,
	\begin{equation}
		\label{eq:proof2_sub3}
		\lambda^{(g+1)} - \lambda^{(g)} \leq \dfrac{\delta}{q}. \notag
	\end{equation}
	Let $q = 0.2642$. Then, 
	\begin{equation} 
		\lambda^{(g+1)} - \lambda^{(g)} \leq \dfrac{\delta}{0.2642}, \notag
	\end{equation}
resulting in,
	\begin{equation} 
	\label{eq: lambda<L}
		\lambda^{(g+1)} - \lambda^{(g)} \leq L, \notag
	\end{equation}
such that $\delta = qL$.\\

\noindent
	Therefore, an insignificant change in population size, specifically when $\lambda^{(g+1)} - \lambda^{(g)} \leq L$, results in $\mu_{w}^{(g+1)} - \mu_{w}^{(g)}$ to be sufficiently small causing a step-size blow-up as established in Theorem \ref{Theorem:Factor>1}.
  
\end{proof}

Thus we conclude that a blow-up in step size occurs due to the formulation of the step size correction mechanism when the algorithm is nearing a minimum (Theorem \ref{Theorem:Factor>1}) and the population size change is insignificant (Theorem \ref{Theorem:Delta_lambda}). In the next section we showcase numerical evidence to support these results. 

	
\subsection{Experimental Evidences} 	
\label{sec:Exp-Setting}

\begin{figure}[!h]
	\centering
	\includegraphics[width=1\linewidth]{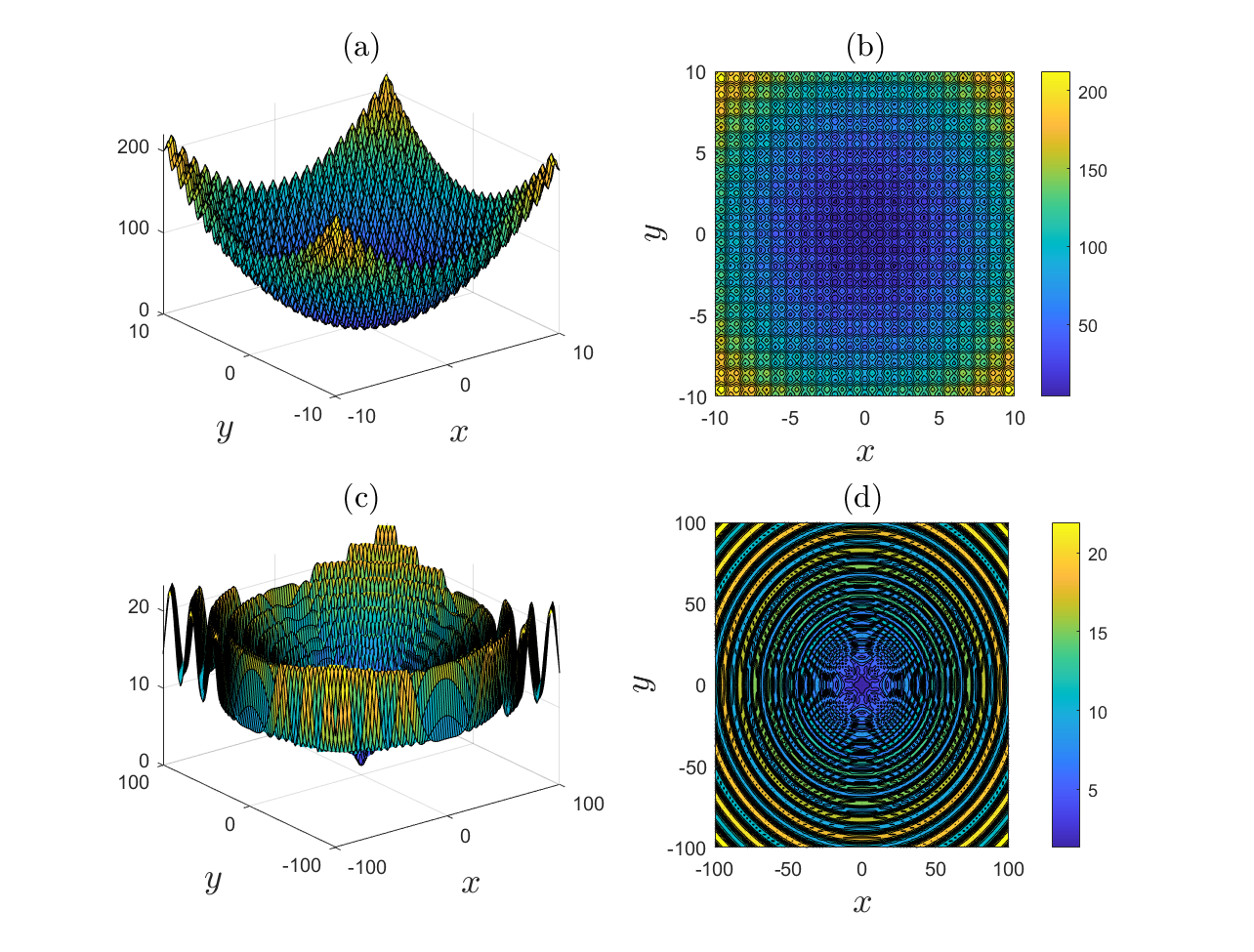}
	\caption{Visualization of the benchmark test functions; Rastrigin function is depicted in the upper panel (a), (b) and Schaffer function is depicted in the lower panel (c) and (d). The global optimum of each function (origin) are shown in the contour plots (b) and (d) respectively.}
	\label{fig:benchmarkgraphs}
\end{figure}

Two benchmark functions were used in this study to design a series of experiments to demonstrate how the experimental results align with our analytical findings. Both benchmark functions studied in this work are recognized for their complex optimization landscapes characterized by multimodality, continuity, non-separability, and nonconvexity \cite{PSA_2}. The analytical form of the two benchmark functions in its $n-$dimensional form in the domain of $x_i \in \left[ a,b\right] $ (see Fig.\ref{fig:benchmarkgraphs}) are given below.
\begin{align} 
	\label{eq:Rastrigin_function}
	f_R(\mathbf{x}) = 
	\sum_{i=1}^{n} \left( x_{i}^{2} + 10 \left( 1 - \cos(2\pi x_{i}) \right) \right),  x_i \in \left[ -10,10\right],
\end{align}

\begin{equation} 
	\label{eq:Schaffer_function}
	f_S(\mathbf{x}) = 
	\sum_{i=1}^{n-1} \left( \left( x_{i}^{2} + x_{i+1}^{2} \right)^{1/4} 
	\left( \sin^{2} \left( 50 \left( x_{i}^{2} + x_{i+1}^{2} \right)^{0.1} \right) + 1 \right) \right),  x_i \in \left[ -100,100\right],
\end{equation}
where $f_R$ and $f_S$ are respectively represent Rastrigin and Schaffer functions.

\subsubsection{Population size ($\lambda$) change  } \label{population change-exp}
The significance level $L$ (refer Eq. \eqref{eq: lambda<L}) for the differences between $\lambda_r^{(g)}$ and $\lambda_r^{(g+1)}$ that leads to a step-size blow-up was empirically determined through two controlled experiments (Experiments 1 and 2). To verify whether the step-size correction plays a role in $\sigma$ blow-up, the two controlled experiments were conducted under the conditions: with step-size correction (Experiment 1) and without step-size correction (Experiment 2).\\

 Experiment 1 and 2 were conducted for Rastrigin function and these results are depicted in Fig. \ref{fig:R-Exp-1.1,1.2,1.3} and \ref{fig:R-Exp-1.1,1.2,1.3_diff}. Experiment 1 systematically increased the population size $\lambda_r^{(g+1)}$ by known integer values as given in Eq. \eqref{4.3a}. Corresponding $\lambda^{(g)}$ and step-size values $\sigma^{(g)}$ are given in Fig. \ref{fig:R-Exp-1.1,1.2,1.3} (a). Differences in step-size $\Delta\sigma^{(g)}$ at each generation are given in Fig. \ref{fig:R-Exp-1.1,1.2,1.3_diff} (a). Based on these results for 2D Rastrigin function, $\Delta\sigma^{(g)}$ starts blowing up after generation $6$ when $\Delta\sigma^{(g)}\geq6$. To verify the results, the same experiement was repeated by decreasing the population size values as given in Eq. \eqref{4.3b}. Similarly, $\sigma^{(g)}$, $\Delta\sigma^{(g)}$ were recorded (see Fig. \ref{fig:R-Exp-1.1,1.2,1.3} (b), Fig. \ref{fig:R-Exp-1.1,1.2,1.3_diff} (b)). Accordingly, results verify that the step-size blow-up takes place when $\Delta\sigma^{(g)}\geq6$. Therefore it can be stated that $6$ is a significant level for the difference in population size, $\lambda^{(g+1)} - \lambda^{(g)}$ which causes step-size blow-up. This further clarifies that step-size blow up when the change in $\mu_w$ falls below $\delta = 1.5852$ according to Theorem \ref{Theorem:Factor>1}.
	\begin{subequations} \label{eq:lambda_update}
		\begin{align}
			\lambda_r^{(g+1)} &= \lambda_r^{(g)} + \sum_{i=1}^{g}i, \tag{4.3a} \label{4.3a} \\
			\lambda_r^{(g+1)} &= \lambda_r^{(g)} - \sum_{i=1}^{g}i. \tag{4.3b} \label{4.3b}
		\end{align}
	\end{subequations}

 	\begin{figure}[h!]
 		\centering
 		\includegraphics[width=1\linewidth]{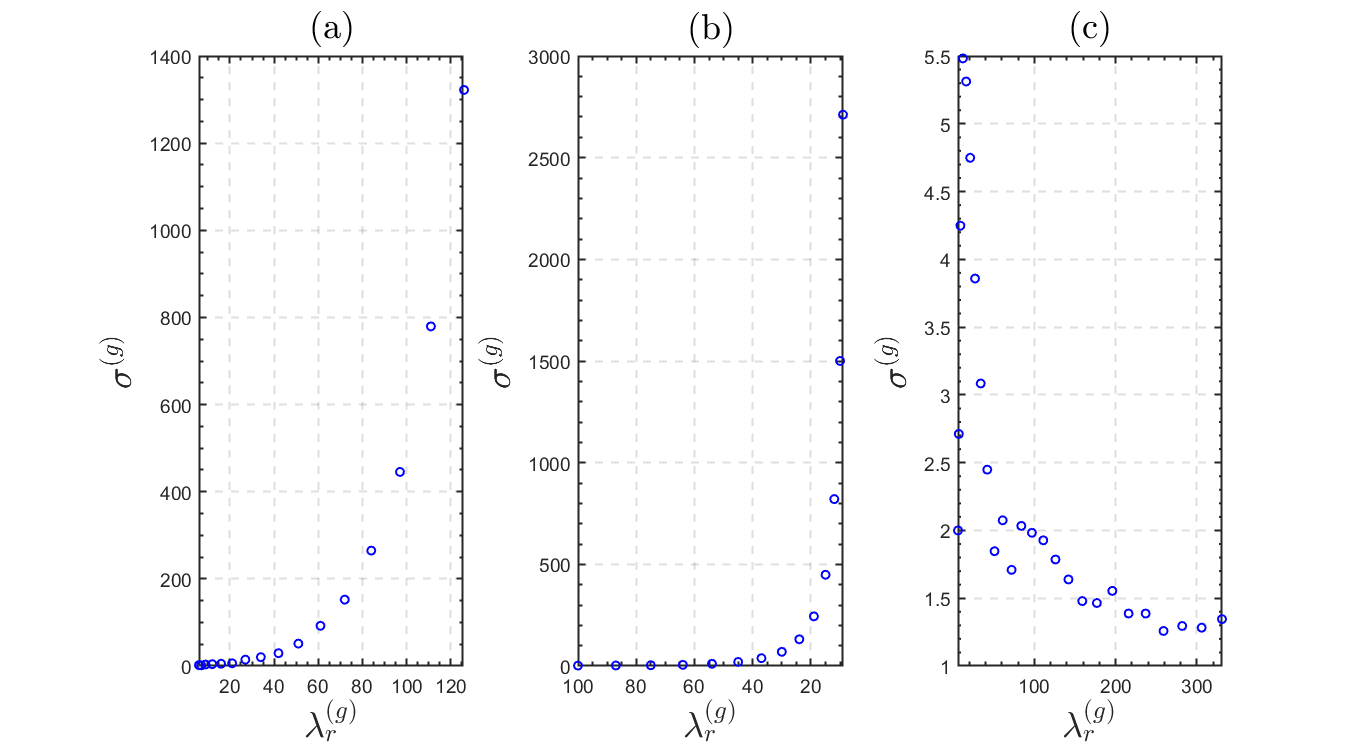}
 		\caption{Summary of the results across experiments 1 and 2 for the 2D Rastrigin function. Panels (a) and (b) depict the blow-up of step-size in Experiment 1, starting at generation \(6\rightarrow7\), with increasing and decreasing population sizes respectively. (c) shows the results of Experiment 2, highlighting the rapid decrease in step-size when the correction step is removed.}
 		\label{fig:R-Exp-1.1,1.2,1.3}
 	\end{figure}
 	
 	\begin{figure}[h!]
 		\centering
 		\includegraphics[width=1\linewidth]{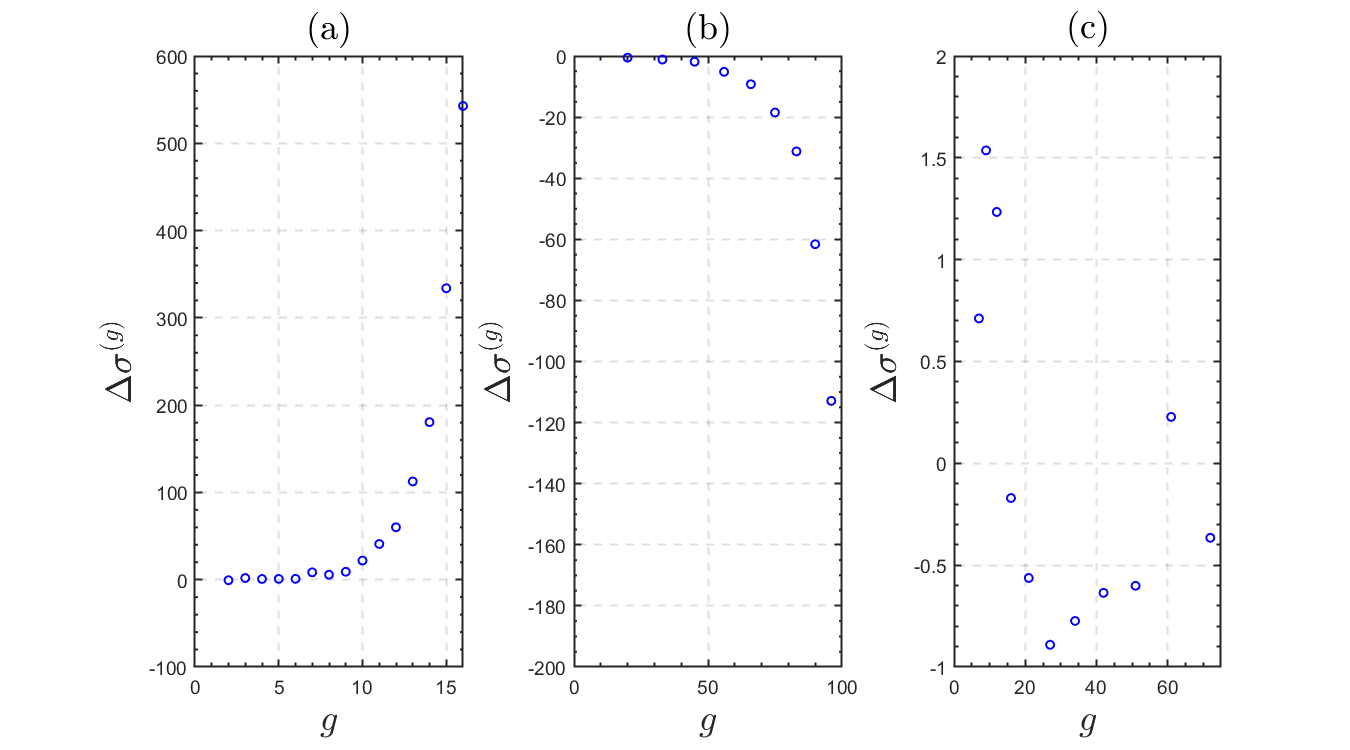}
 		\caption{Summary of the results across experiments 1 and 2 for the 2D Rastrigin function. Panels (a) and (b) depict the blow-up of step-size in terms of the step-size difference ($\Delta\sigma^{(g)}$) in Experiment 1, with increasing and decreasing population sizes respectively. (c) shows the results of Experiment 2, in terms of the step-size difference, highlighting the rapid decrease in step-size when the correction step is removed.}
 		\label{fig:R-Exp-1.1,1.2,1.3_diff}
 	\end{figure}

 	In experiment 2, the general PSA-CMA-ES algorithm was performed similar to experiment 1 without the step-size correction. Corresponding results of $\sigma^{(g)}$ against $\lambda^{(g)}$ and $\Delta \sigma^{(g)}$ against $g$ are depicted in Figs.  \ref{fig:R-Exp-1.1,1.2,1.3} and \ref{fig:R-Exp-1.1,1.2,1.3_diff} panel (c) respectively. In these results the step-size did not blow up as the population size $\lambda_{r}$ increased. These results were obtained for both Rastrigin and Schaffer functions in Eqs. \eqref{eq:Rastrigin_function}, \eqref{eq:Schaffer_function}. This confirmed that the step-size blow up observed in general PSA-CMA-ES was caused solely by the step-size correction mechanism, as the step-size rapidly decreases without it. However, at the same time these results show the importance of the step-size correction as well. As seen in Fig. \ref{fig:R-Exp-1.1,1.2,1.3} (c), the step-size decrease rapidly with the increase of population size resulting in higher likelihood of premature convergence. Therefore it is important to maintain the step-size correction in a safe scale. \\

\subsubsection{Optimizing step-size correction} \label{sec: optimize}

In the previous section, experiment 2 demonstrated completely removing the step-size correction fail to improve performance of the general PSA-CMA-ES algorithm (see Fig. \ref{fig:R-Exp-1.1,1.2,1.3} (c)). Thus it is important to maintain this correction at a safe scale effectively balancing exploration and exploitation while minimizing computational complexity. To toptimize this balance, we scaled the step-size correction of the existing PSA-CMA-ES algorithm with a parameter $\kappa \in \left( 0,1\right) $. The optimal value of $\kappa$ was determined using experimentally, specifically for 2D Rastrigin and Schaffer functions.\\

The PSA-CMA-ES was run with the initial mean, \(\textbf{m}^{(0)}\) that was uniformly sampled from a predefined interval (\(\mathrm{I}^{(0)}\)) to measure the performance metrics, (1) CPU time until termination (CPU time), (2) accuracy of the optimum reached $|f^*-f|$, and (3) average number of function evaluations per generation ($f_N$). At each run, step-size correction was scaled by a factor of $\kappa$ where $\kappa $ was varied from 0 to 1 with 0.1 increments. Here, $\kappa = 0$ is equivalent to having no step-size correction and $1$ is equivalent to the original PSA-CMA-ES. For these individual runs, two stopping criteria were considered; either (1) when the global minimum was reached within a tolerance of \(10^{-2}\) or (2) completed a predefined maximum number of generations - whichever occurred first. The inclusion of the second criterion was necessary to prevent the algorithm from becoming trapped in loops, which occurred when the mean value crossed the problem's search boundary as a result of step-size blow-up caused by step-size correction. The values of each performance metric are shown in Fig. \ref{fig:kappaexperiment}. To find the optimum scaling $\kappa$, an overall performance metric, $S_f$ was calculated by giving uniform weights across each metric,
\begin{align}
	S_f = \text{CPU time} + |f^*-f|+f_N. \label{eq:S_f_performance_metric}
\end{align}
 The detailed values of these results are presented in appendix (see Table  \ref{table:RastriginEmpricalAllruns} and \ref{table:SchafferEmpricalAllruns}). For both functions the optimum $\kappa$ was found to be $0.5$. However, it should be noted that this value may change for other functions.

\begin{figure}[!hbt]
	\centering
	\includegraphics[width=1\linewidth]{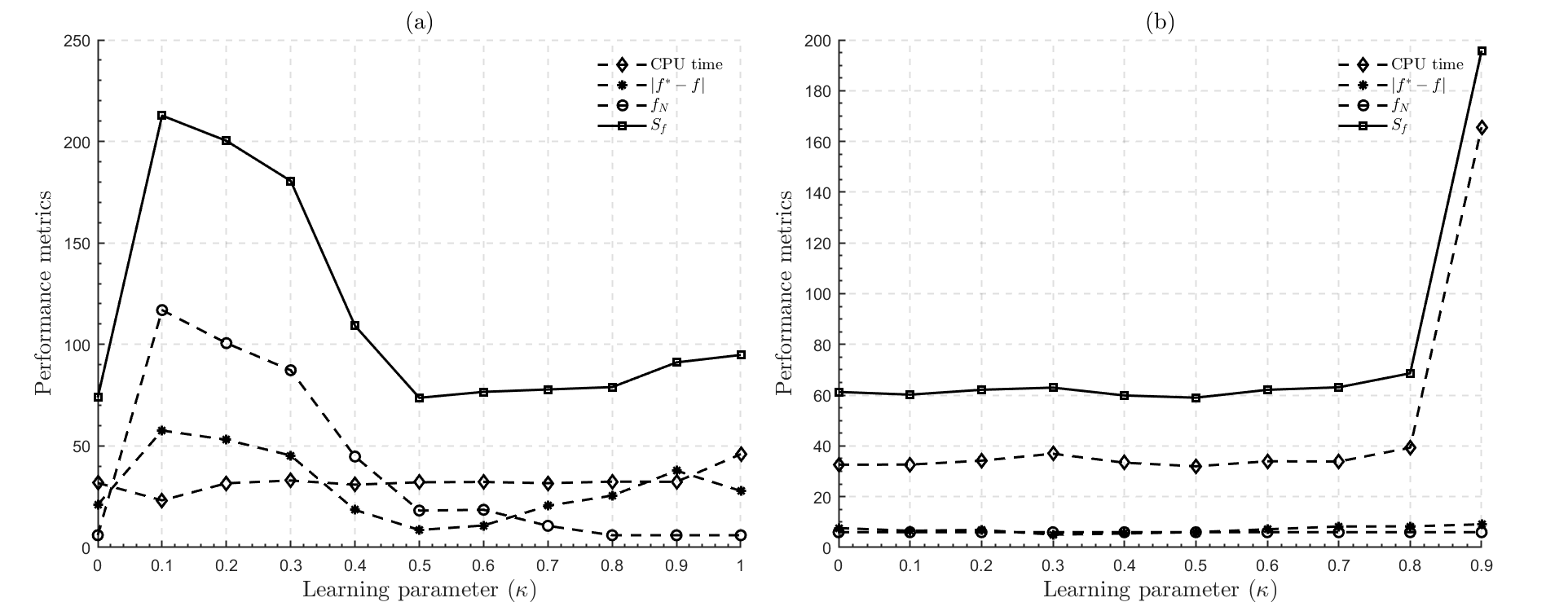}
	\caption[Performance for each \(\kappa\)]{
		The values of observed performance metrics for each \(\kappa\) value ranging from $0$ to $1$ at $0.1$ increments. Panels (a) and (b) show the performance on the Rastrigin and Schaffer functions respectively. The dotted lines are drawn for the ease of following the trend and does not represent a connection. Four performance metrics were considered (1) CPU time until termination represented by diamonds, (2) accuracy of the optimum reached $|f^*-f|$ represented by asterisks, (3) average number of function evaluations per generation ($f_N$) represented by circles and finally (4) the overall performance metric $S_f$ (refer Eq. \ref{eq:S_f_performance_metric}) represented by squares. For comparison, the number of generations completed was fixed at $15$ for each independent run. Notably, the lowest value for all metrics was observed at \(\kappa = 0.5\), corresponding to the minimum $S_f$ value.}
	\label{fig:kappaexperiment}
\end{figure}

\subsection{Reformulated step-size correction}
       \label{sec:Reformulation}
This section introduces the reformulated step-size correction mechanism addressing the shortcomings of the general PSA-CMA-ES argued in the previous sections. The step-size correction is reformulated to improve the algorithm's convergence when nearing the global minimum by preventing the step-size blow-up discussed, while simultaneously improving exploration of the search space and
 maintaining low computational complexity. \\

 In PSA-CMA-ES, the step-size correction applied uniformly across all cumulative step-size correlations; (a) anti-correlated, (b) corelated,  and (c) uncorrelated (see Eq. \eqref{eq:5}) \cite{PSA_2,PSA_3}. However, in the reformulation, step-size correction is applied selectively taking into account the step-size adaptation. That is, the correction is only applied when the cumulative steps are anti-correlated (Alg. \ref{alg:CMA-ES} line 19), i.e when the adaptation demands the algorithm to decrease the step-size when nearing an optimal.  In that, the correction is applied conditionally based on population size change.  Specifically, the step-size is scaled only if the population size change is insignificant to avoid blow-up as illustrated in previous sub-sections. This scaling mechanism was motivated by analytical findings (Theorem \ref{Theorem:Factor>1}). It showed step-size blow-up occurs when $\lambda^{(g+1)} - \lambda^{(g)} \leq {L}$, since the resulting change in $ \mu_w$ remains within the small bound $\delta$ (Theorem \ref{Theorem:Factor>1}). Hence, the scaling mechanism in the reformulation prevents excessive growth in the step-size that occurs when the population size change is below ${L}$. If this criterion is met, the step-size correction is scaled in the reformulation by a scaling parameter $\kappa$, allowing only a fraction of the step-size correction. When the population size change was not insignificant i.e. $\lambda^{(g+1)} - \lambda^{(g)} > {L}$, the original step-size correction was applied without scaling. The proposed reformulation is outlined in Algorithm \ref{alg:Reformulation:Case_1}. The new PSA-CMA-ES algrithm with the reformulation is given in Algorithm \ref{alg:CMA-ES}. The definitons of the parameters in the reformormulated PSA-CMA-ES are given in Table \ref{tab:definitions}. The input parameters of the reformulated algorithm (Alg. \ref{alg:CMA-ES}) are $\mathbf{m}^{(0)}, \sigma^{(0)}, \kappa, \text{ and } {L}$. In that, $\mathbf{m}^{(0)}, \sigma^{(0)}$ are the mean and step-size of the initial candidate population (see Table \ref{tab:init}). The parameters $\kappa$ was experimentally found and discussed in Section \ref{sec: optimize}.

\SetKwInput{KwInput}{Input}
\SetKwInput{KwOutput}{Set}
\SetKwInput{KwResult}{Initialize}

\begin{algorithm}[h!]
	\caption{Step-Size Correction Reformulation}\label{alg:Reformulation:Case_1}
	\eIf{\(||\textbf{p}_{\sigma}|| \geq E\left[||N(\textbf{0},\textit{I})||\right]\)}
	{
		\(\sigma_c^{(g+1)} = \sigma^{(g+1)}\),
	}{
		\eIf{\(|\lambda^{(g+1)} - \lambda^{(g)}|< {L}\)}
		{\(\sigma_c^{(g+1)} = \sigma^{(g+1)} \cdot  \kappa \, \frac{\rho\left(\lambda^{(g+1)}_{r}\right)}{\rho\left(\lambda^{(g)}_{r}\right)}\),}
		{\(\sigma_c^{(g+1)} = \sigma^{(g+1)} \cdot \frac{\rho\left(\lambda^{(g+1)}_{r}\right)}{\rho\left(\lambda^{(g)}_{r}\right)}\).}
	}
\end{algorithm}
\begin{table}[ht!]
	\centering
	\caption{Definition of parameters used in the reformulated algorithm (Alg. \ref{alg:CMA-ES} ) } \label{tab:definitions}
	\begin{tabular}{ll}
		Parameter    & Definition                                                                    \\ \hline
		$\mathbf{m}$ & Mean of the distribution                                                       \\
		$\sigma$     & Step-size of the distribution                                                  \\
		$\sigma_c$ & Step-size after correction \\
		$\mathrm{I}^{(0)}$  & Initial interval from which $\mathbf{m}^{(0)}$ and $\sigma^{(0)}$ are sampled \\
		$\kappa$     & Scaling parameter for the reformulated step-size correction                    \\
		${L}$ & Significance level for the population size change       \\
		$\mu$ & Best fitness $\lfloor \lambda_r/2 \rfloor$ population (parent population)\\
		$w_i$ & Weights of $i-$th best candidate \\
		$g_\text{max}$ & Maximum number of generations completed in one algorithm run\\
		$\mu_w$ & Variance effective selection mass
	\end{tabular}
\end{table}

\begin{table}[ht!]
	\centering
	\caption{Details of the initial candidate populations} \label{tab:init}
	\begin{tabular}{lll}
		Parameter       & Rastrigin function                   & Schaffer function                  \\ \hline
		$ \mathrm{I}^{(0)} $  & $ [1,5]$     & $[10,100]$   \\
		$\mathbf{m}^{(0)}$  & $\sim$ Uni$ \left(\mathrm{I}^{(0)} \right) $              & $\sim$ Uni$ \left(\mathrm{I}^{(0)} \right) $       \\
		$\sigma^{(0)}$ & $\mathrm{I}^{(0)}/2 $ & $\mathrm{I}^{(0)}/2 $ 
	\end{tabular}
\end{table}
In summary, this reformulation ensured the step-size correction aligned with the direction suggests by the step-size adaptation mechanism in the PSA-CMA-ES. If the adaptation indicated a need for continued exploration (i.e., step-size increase with $||\textbf{p}_{\sigma}||\geq \text{E}[||N(\textbf{0},\textbf{I})||]$), then the  step-size correction step was omitted. On the other hand if the adaptation suggested a convergence (i.e., step-size decrease) then the step-size correction was applied—but scaled down if the population size change was insignificant. This selective approach prevented both premature convergence and continuous step-size growth leading to a blow-up. In the next section, we demonstrate how this reformulated PSA-CMA-ES algorithm outperformed the general PSA-CMA-ES algorithm.

\subsection{Performance of reformulated PSA-CMA-ES algorithm}
	\label{sec:Implementation_Results}
	
In this section, we showcase the performance of the reformulated PSA-CMA-ES algorithm (Alg. \ref{alg:CMA-ES}) in comparison to the general PSA-CMA-ES algorithm (Section \ref{generalPSA}). The two 2-dimensional benchmark functions ($f$): Rastrigin and Schaffer, borrowed from \cite{PSA_2} (see Eqs. \eqref{eq:Rastrigin_function},\eqref{eq:Schaffer_function}) were tested to evaluate the reformulation. These evaluations are conducted based on three performance metrics; (1) CPU time until termination (CPU), (2) accuracy of the optimum reached $|f^*-f|$, and (3) average number of function evaluations per generation ($f_N$). \\

	
	 The global minimum, \(\mathbf{x}^{*}\), of the Rastrigin function is at \((0,0)\) and \(f(\mathbf{x}^{*}) = 0\). The algorithm was initialized using the values in Table \ref{tab:init} following \cite{PSA_3}. To compare with the original PSA-CMA-ES algorithm, the reformulation was limited to \(20\) generations on the 2D Rastrigin function.	The three performance metrics for the two algorithms are compared in Fig. \ref{fig:performancecomparison}. The results show a significant reduction in CPU time for the reformulated PSA-CMA-ES algorithm compared to the general PSA-CMA-ES (average of 20 runs were 33.1779 and 116.5880 respectively), because it is prone to getting stuck in loops, which would lead to even higher CPU times if run for additional generations. In addition to better computational efficiency, the reformulated algorithm consistently produced optimum value closer to \(f(\mathbf{x}^{*}) = 0\) indicating superiority in convergence ($|f^*-f| $ is 34.0996 and 12.8041 for general and reformulation respectively). Although the reformulation leads to a higher number of function evaluations per generation (approximately average of 327 for reformulation and 6 for general)—due to increased population sizes from selective step-size corrections—it nevertheless achieves superior performance in terms of CPU time.\\

Similarly, the experiments were carried out for the Schaffer function with initial values provided in Table \ref{tab:init} following \cite{PSA_3}. The function's global minimum \(\mathbf{x}^{*}\) is located at \((0,0)\) with \(f(\mathbf{x}^{*}) = 0\). To compare with the original PSA-CMA-ES algorithm, the reformulation  was limited to \(10\) generations. This value was less compared to Rastrigin due to its complexity of the Schaffer. Fig. \ref{fig:performancecomparison} illustrates the results of  performance metrics obtained from 20 independent runs with each run completing \(10\) generations. Similar to results on the 2D Rastrigin function, the reformulation shows a significant reduction in CPU time (average of 20 runs were $22.4839$ and $57.9054$ respectively) highlighting the drawback of getting stuck in loops in the general PSA-CMA-ES. Additionally, the reformulation consistently produced optimum values closer to \(f(\mathbf{x}^{*}) = 0\), indicating superior convergence ($|f^*-f| $ is $9.3576$ and $7.1450$ for general and reformulation respectively). The number of function evaluations was similar for both the general PSA-CMA-ES and the reformulation, indicating that the selective step-size correction did not lead to any increase in population size (approximately average of 6 for both algorithms).\\
 
Considering the overall performance of both benchmark functions, the reformulated algorithm acts superior to the general PSA-CMA-ES with respect to the CPU time and convergence. Although the function evaluations may increase in the reformulation, it supports the convergence to the global optimum. However, it is important to trace the function evaluations when the complexity of the problem demands increased number of generations in a run.

	\begin{figure}[!h]
		\centering
		\includegraphics[width=1\linewidth]{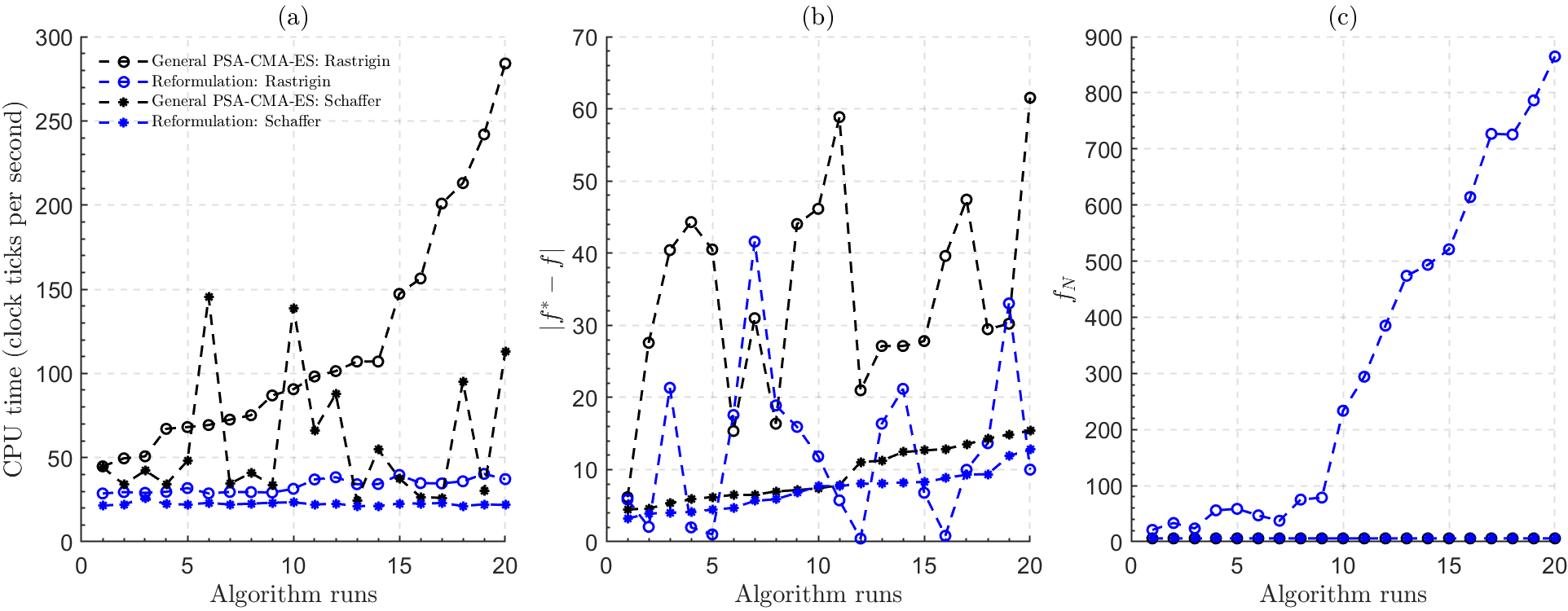}
		\caption[Performance comparison of the Reformulation against General PSA-CMA-ES ]{Performance comparison of the Reformulation against General PSA-CMA-ES on the 2D Rastrigin and 2D Schaffer functions over 20 independent runs. For the Rastrigin function, the algorithm completed 20 generations per run, while for the Schaffer function, it completed 10 generations per run. Results on the Reformulation and General PSA-CMA-ES are indicated in blue and black respectively. Circles are used to represent results on the Rastrgin function wile asterisk marks are used for the Schaffer function. The values obtained are connected by dotted lines to easily follow the trend and does not represent a connection.}
		\label{fig:performancecomparison}
	\end{figure}

	\section{Conclusion and Future Works}
	\label{sec:Conclusion}
	
	Multimodal optimization relies on a balance between exploration to locate promising regions and exploitation to refine solutions within them. The PSA-CMA-ES algorithm achieves this balance by dynamically adjusting population size. While it performs well on well-structured problems, it struggles with weakly-structured ones due to ineffective step-size control, leading to step-size blow-up and poor convergence near the global optimum.\\
	
	In this study, we introduced a reformulation of the step-size correction mechanism in the PSA-CMA-ES algorithm to improve global search and convergence on weakly-structured multimodal problems (refer Algorithm \ref{alg:Reformulation:Case_1} for the reformulation). The goal was to improve convergence at the global optimum while reducing CPU time complexity. To achieve this, we analytically studied the step-size correction mechanism of general PSA-CMA-ES and identified the cause for the step-size blow-up (Theorem \ref{Theorem:Factor>1}) and showed the existence of a significance level for population size change (refer Theorem \ref{Theorem:Delta_lambda} and Figs \ref{fig:R-Exp-1.1,1.2,1.3}, \ref{fig:R-Exp-1.1,1.2,1.3_diff}). The reformulation was guided by the insights gained from this analysis. The reformulated algorithm outperformed the general PSA-CMA-ES by achieving better optimum values with reduced CPU time. This improvement was demonstrated through experiments on the 2D Rastrigin and Schaffer functions (refer Fig. \ref{fig:performancecomparison}). Notably, while CPU time is inherently lower in two dimensions, the results suggest that as the problem complexity increases in higher dimensions, the reformulation continues to offer significant reductions in CPU time, along with improved convergence at the optimum, compared to the general PSA-CMA-ES.\\

	Further, the reformulation better exploited the population size adaptation by increasing the population size which in turn improved the search space exploration. This increase led to increases in the number of function evaluations per generation as especially seen on the Rastrigin function (Fig. \ref{fig:performancecomparison} (c)), but nevertheless the CPU time taken was still lower than that of general PSA-CMA-ES. Moreover, the reformulation allowed for better exploration by aligning step-size changes more closely with the direction suggested by cumulative step-size adaptation. The convergence was improved by including a selective version of the step-size correction step. Together, these changes led to improved performance final optimum reached and computational efficiency. \\

	One noticeable unfavorable feature of the introduced reformulation was the continuous increase in population size, especially on the Rastrigin function, even after convergence had begun. This increase in population size led to an increased number of function evaluations. Ideally, once the algorithm identifies an optimum and the step-size begins to decrease, the population size should also reduce with the focus shifting from exploration to exploitation. Therefore, introducing a population size scaling mechanism presents a valuable direction for future work. This refinement could prevent unnecessary function evaluations during convergence and further improve overall efficiency. The promising performance of the reformulation on weakly-structured multimodal problems indicates the need for continued investigation in this avenue.  \\

	\section*{Author contributions}
	
	\textbf{Chandula Fernando:} Conceptualization, Methodology, Software, Validation, Investigation, Writing - original draft, Visualization.
	\textbf{Kushani De Silva:} Conceptualization, Methodology, Writing - Original Draft, Writing - Review \& Editing, Supervision
		
%
%
%
%


	\bibliographystyle{apalike}
	\bibliography{Bibliography.bib}
	
	\newpage
		\section{Appendix}
	
	\RestyleAlgo{ruled}
	\LinesNumbered
	\begin{algorithm}[ht!]
		\caption{PSA-CMA-ES with Reformulated Step-size Correction}\label{alg:CMA-ES}
		\KwInput{\(\textbf{m}^{(0)} \in \mathbb{R}^{n}, \sigma^{(0)} \in \mathbb{R}_{+}\), $\kappa$, $\mathrm{L}$, $a,b$, $g_{\text{max}}, \epsilon$}
		\KwOutput{\(c_{m} = 1, \sum_{i=1}^{\mu}w_i = 1, \alpha = 1.4, \beta = 0.4, \lambda_{\text{min}} = \lambda_{\text{def}}, \lambda_{\text{max}} = 512 \, \lambda_{\text{def}}\)}
		\KwResult{\(\textbf{C}^{(0)} = \textbf{I}, \textbf{p}_{c}^{(0)} = \mathbf{0}, \textbf{p}_{\sigma}^{(0)} = \mathbf{0}, \textbf{p}_{\theta}^{(0)} = \mathbf{0}, \gamma_{c}^{(0)} = 0, \gamma_{\sigma}^{(0)} = 0, \gamma_{\theta}^{(0)} = 0, \lambda^{(0)} = \lambda_{r}^{(0)} = \lambda_{\text{def}}, g=0\)}
		
		\While{ $g<g_{\text{max}}$ or $\left|f\left( \textbf{m}^{(g+1)}\right) - f\left( \textbf{m}^{(g)}\right) \right| >\epsilon $}{
			\footnotesize\ttfamily{//(re-)compute parameters depending on \(\lambda_r\)}\\
			\(\mu \leftarrow \lfloor \lambda_{r}/2 \rfloor\)\\
			\(w_{i} \leftarrow \dfrac{\ln(\frac{\lambda+1}{2}) - \ln(i)}{\sum_{i=1}^{\mu} [\ln(\frac{\lambda+1}{2}) - \ln(i)]} \text{ for } i = 1,\dots,\mu; \text{ and } 0 \text{ otherwise} \)\\
			\(\mu_{w} \leftarrow 1/\sum_{i=1}^{\mu} w_{i}^{2}\)\\   
			\(c_{\sigma} \leftarrow (\mu_{w} + 2)/(n + \mu_{w} + 5)\)\\   
			\(d_{\sigma} \leftarrow 1 + 2\max(0,\sqrt{(\mu_{\text{eff}}-1)/(n+1)} - 1) + c_{\sigma}\)\\   
			\(c_{c} \leftarrow (4 + \mu_{\text{eff}}/n)/(n+4+2\mu_{\text{eff}}/n)\)\\ 
			\(c_{1} \leftarrow 2/((n+1.3)^{2} + \mu_{\text{eff}})\)\\
			
			\(\textbf{x}^{(g+1)}_i \sim \textbf{m}^{(g)} + \sigma^{(g)}N(\textbf{0},\textbf{C}^{(g)}), \text{ for } i=1,...,\lambda.\)\\
			\footnotesize\ttfamily{Perform CMA-ES iteration (\ref{eq:2}), (\ref{eq:5}) and (\ref{eq:cm_update}) of Section \ref{sec:PSA-CMA-ES}}\\
			
			\(\textbf{m}' \leftarrow \textbf{m}, \textbf{C}' \leftarrow \textbf{C}, \sigma' \leftarrow \sigma\);  \footnotesize\ttfamily{//Keep old values}\\
			
			\footnotesize\ttfamily{Update evolution paths  (\ref{eq:4}), (\ref{eq:pc}) and (\ref{eq:p_theta}) of Section \ref{sec:PSA-CMA-ES}}\\
			
			\(\gamma^{(g+1)}_{\sigma} = (1-c_{\sigma})^{2}\gamma^{(g)}_{\sigma} + c_{\sigma}(2 - c_{\sigma})\); \footnotesize\ttfamily{//Update normalization factors}\\
			\(	\gamma^{(g+1)}_{c} = (1-c_{c})^{2}\gamma^{(g)}_{c} + h^{(g+1)}_{\sigma} c_{c}(2 - c_{c})\)\\
			\(\gamma^{(g+1)}_{\theta} = (1 - \beta)^{2}\gamma^{(g)}_{\theta} + \beta(2 - \beta)\)\\
			
			\footnotesize\ttfamily{Update population size (\ref{eq:pop-size_update}), and (\ref{eq:pop-size_boundary}) of Section \ref{sec:PSA-CMA-ES}}\\
			
			\footnotesize\ttfamily{//Reformulated step-size correction}\\
			\eIf{\(||\textbf{p}_{\sigma}|| \geq \mathbb{E}\left[||\mathcal{N}(\textbf{0},\textit{I})||\right]\)}{
				\(\sigma^{(g+1)} = \sigma^{(g+1)}\);
			}{
				\eIf{\(|\lambda^{(g+1)} - \lambda^{(g)}|< \mathrm{L}\)}{
					\(\sigma^{(g+1)} = \sigma^{(g+1)} \cdot \kappa \frac{\rho\left(\lambda^{(g+1)}_{r}\right)}{\rho\left(\lambda^{(g)}_{r}\right)}\)
				}{
					\(\sigma^{(g+1)} = \sigma^{(g+1)} \cdot \frac{\rho\left(\lambda^{(g+1)}_{r}\right)}{\rho\left(\lambda^{(g)}_{r}\right)}\); \\
					\footnotesize\ttfamily{ Here \(\rho\) is a function of \(\lambda^{(g)}_{r}\) and \(\rho(\lambda^{(g)}_{r}) = \frac{\left(-\sum_{i=1}^{\lambda_{r}}w_i \mathbb{E}[\mathcal{N}_{i:\lambda_{r}}]\right)\cdot n \cdot \mu_{\text{w}}}{n-1+\left(\left(-\sum_{i=1}^{\lambda_{r}}w_i \mathbb{E}[\mathcal{N}_{i:\lambda_{r}}]\right)^{2}\mu_{\text{w}}\right)}\)}
				}
			}
			\(g \leftarrow g+1\)
			
		}
	\end{algorithm}

\begin{table}[hbt!]
	\centering
	\caption{Implementation Results Comparison for Rastrigin Function}
	The columns give the CPU time taken, the minimum value found by the algorithm \textit{(val)}, average number of function evaluations per generation \textit{(Avg \# f.eval)}, and number of generations completed \textit{(\# gens)} for the general PSA-CMA-ES algorithm and the Reformulation on the Rastrigin function.
	\label{table:RastriginComparison}
	\resizebox{0.99\textwidth}{!}{
		\begin{tabular}{|l |l  |l |l |l |l  |l |l |l|} \hline  
			Function & \multicolumn{4}{|c|}{Rastrigin} & \multicolumn{4}{|c|}{Rastrigin} \\ 
			Algorithm & \multicolumn{4}{|c|}{General PSA-CMA-ES} & \multicolumn{4}{|c|}{Reformulation} \\ \hline 
			
			Run \# & CPU time & val & Avg \# f.eval & \# gens & CPU time & val & Avg \# f.eval & \# gens \\ \hline 
			
			1 & 44.7609 & 6.1721 & 6 & 20 & 28.6819 & 5.8601 & 21.14 & 20 \\ 
			2 & 49.5271 & 27.5756 & 6 & 20 & 29.4904 & 2.0655 & 33.38 & 20 \\ 
			3 & 50.7356 & 40.4244 & 6 & 20 & 29.0622 & 21.3385 & 23.71 & 20 \\ 
			4 & 67.0880 & 44.3008 & 6 & 20 & 29.6653 & 1.9913 & 55.95 & 20 \\ 
			5 & 68.1091 & 40.5126 & 6 & 20 & 31.8646 & 1.0115 & 58.67 & 20 \\ 
			6 & 69.3169 & 15.3446 & 6 & 20 & 28.7257 & 17.5701 & 47.00 & 20 \\ 
			7 & 72.5328 & 30.9816 & 6 & 20 & 29.5821 & 41.5951 & 37.43 & 20 \\ 
			8 & 75.1418 & 16.3570 & 6 & 20 & 29.4873 & 18.8904 & 75.14 & 20 \\ 
			9 & 86.9018 & 44.0385 & 6 & 20 & 29.2569 & 15.9245 & 78.71 & 20 \\ 
			10 & 90.5863 & 46.1390 & 6 & 20 & 31.4461 & 11.8287 & 233.48 & 20 \\ 
			11 & 98.1587 & 58.8558 & 6 & 20 & 36.9983 & 5.7348 & 294.00 & 20 \\ 
			12 & 101.3522 & 20.9890 & 6 & 20 & 38.3244 & 0.4477 & 385.10 & 20 \\ 
			13 & 107.0312 & 27.1279 & 6 & 20 & 34.1822 & 16.3639 & 473.90 & 20 \\ 
			14 & 107.0312 & 27.1279 & 6 & 20 & 34.2596 & 21.1906 & 493.43 & 20 \\ 
			15 & 147.2240 & 27.8108 & 6 & 20 & 39.6896 & 6.7894 & 521.10 & 20 \\ 
			16 & 156.3896 & 39.6093 & 6 & 20 & 34.8164 & 0.8307 & 614.19 & 20 \\ 
			17 & 200.8323 & 47.4182 & 6 & 20 & 34.6502 & 9.9742 & 726.81 & 20 \\ 
			18 & 213.0240 & 29.4511 & 6 & 20 & 35.8938 & 13.6550 & 725.62 & 20 \\ 
			19 & 241.9549 & 30.2180 & 6 & 20 & 40.2964 & 33.0261 & 786.43 & 20 \\ 
			20 & 284.0622 & 61.5379 & 6 & 20 & 37.1841 & 9.9931 & 864.67 & 20 \\ \hline
			Average & 116.5880 & 34.0996 & 6 & 20 & 33.1779 & 12.8041 & 327.49 & 20 \\ \hline
		\end{tabular}
	}
\end{table}

\begin{table}[hbt!]
	\centering
	\caption{Implementation Results Comparison for Schaffer Function}
	The columns give the CPU time taken, the minimum value found by the algorithm \textit{(val)}, average number of function evaluations per generation \textit{(Avg \# f.eval)}, and number of generations completed \textit{(\# gens)} for the general PSA-CMA-ES algorithm and the Reformulation on the Schaffer function.
	\label{table:SchafferComparison}
	\resizebox{0.99\textwidth}{!}{
		\begin{tabular}{|l |l  |l |l |l |l  |l |l |l|} \hline  
			Function & \multicolumn{4}{|c|}{Schaffer} & \multicolumn{4}{|c|}{Schaffer} \\ 
			Algorithm & \multicolumn{4}{|c|}{General PSA-CMA-ES} & \multicolumn{4}{|c|}{Reformulation} \\ \hline 
			
			Run \# & CPU time & val & Avg \# f.eval & \# gens & CPU time & val & Avg \# f.eval & \# gens \\ \hline 
			
			1 & 44.2190 & 4.4849 & 6 & 10 & 21.4575 & 3.2307 & 6.00 & 10 \\ 
			2 & 33.9783 & 4.5952 & 6 & 10 & 22.2347 & 3.9037 & 6.00 & 10 \\ 
			3 & 42.2517 & 5.3593 & 6 & 10 & 26.0655 & 4.0299 & 6.00 & 10 \\ 
			4 & 34.1905 & 5.8829 & 6 & 10 & 22.3545 & 4.1181 & 6.00 & 10 \\ 
			5 & 48.0390 & 6.2027 & 6 & 10 & 22.1349 & 4.4548 & 6.00 & 10 \\ 
			6 & 145.5743 & 6.4844 & 6 & 10 & 23.0961 & 4.6786 & 6.00 & 10 \\ 
			7 & 34.4582 & 6.5036 & 6 & 10 & 22.1498 & 5.6671 & 6.00 & 10 \\ 
			8 & 40.8872 & 6.9750 & 6 & 10 & 22.6180 & 5.8858 & 6.00 & 10 \\ 
			9 & 33.7059 & 7.1600 & 6 & 10 & 23.2139 & 6.7838 & 6.00 & 10 \\ 
			10 & 138.8508 & 7.4068 & 6 & 10 & 23.4159 & 7.6943 & 6.00 & 10 \\ 
			11 & 66.2165 & 7.8323 & 6 & 10 & 22.1472 & 7.7142 & 6.00 & 10 \\ 
			12 & 87.7007 & 11.0509 & 6 & 10 & 22.6248 & 8.0420 & 6.00 & 10 \\ 
			13 & 24.6335 & 11.2333 & 6 & 10 & 21.1775 & 8.0480 & 6.00 & 10 \\ 
			14 & 54.9015 & 12.4275 & 6 & 10 & 21.2600 & 8.2397 & 6.00 & 10 \\ 
			15 & 37.7076 & 12.6462 & 6 & 10 & 22.6980 & 8.2550 & 6.00 & 10 \\ 
			16 & 26.4832 & 12.8591 & 6 & 10 & 22.7352 & 8.8118 & 6.00 & 10 \\ 
			17 & 26.1057 & 13.4929 & 6 & 10 & 22.9484 & 9.3033 & 6.00 & 10 \\ 
			18 & 94.9910 & 14.2991 & 6 & 10 & 21.2843 & 9.3637 & 6.00 & 10 \\ 
			19 & 30.1169 & 14.8547 & 6 & 10 & 22.2098 & 11.8642 & 6.00 & 10 \\ 
			20 & 113.0954 & 15.4009 & 6 & 10 & 21.8522 & 12.8104 & 6.00 & 10 \\ \hline
			Average & 57.9054 & 9.3576 & 6 & 10 & 22.4839 & 7.1450 & 6.00 & 10 \\ \hline
		\end{tabular}
	}
\end{table}

	\begin{table}[hbt!]
		\centering
		\caption{Performance of general PSA-CMA-ES algorithm for each value of \(\kappa\) on the 2D Rastrigin and Schaffer functions. }
		\label{table:RastriginSchafferKappaMin}
			{\fontsize{12}{15}\normalsize
			\begin{tabular}{l  |l |l |l |l |l} \hline      
				\multicolumn{6}{l}{Rastrigin} \\ \hline 
				Scaling & Average & Average & Average func & Average \# & sum \\ 
				factor & CPU time  & Func val & -tion eval & generation & complexity \\ \hline 
				0 & 31.7680 & 21.0815 & 6.03 & 15 & 73.8838\\  
				0.1 & 23.2047 & 57.6005 & 116.92 & 15 & 212.7208\\  
				0.2 & 31.6248 & 53.0309 & 100.64 & 15 & 200.2963\\ 
				0.3 & 32.9695 & 45.2879 & 87.30 & 15 & 180.5542\\ 
				0.4 & 30.9252 & 18.4903 & 44.78 & 15 & 109.1999\\ 
				0.5 & 32.1098 & 8.5950 & 18.16 & 14.85 & 73.7127\\ 
				0.6 & 32.2534 & 10.8242 & 18.56 & 15 & 76.6401\\ 
				0.7 & 31.6580 & 20.5170 & 10.65 & 15 & 77.8219\\ 
				0.8 & 32.4289 & 25.5464 & 6.02 & 15 & 78.9910\\ 
				0.9 & 32.3471 & 37.8063 & 6.00 & 15 & 91.1534\\  
				1 & 45.9866 & 27.8375 & 6.00 & 15 & 94.8241\\ \hline
				\multicolumn{5}{l}{Minimum} & 73.7127\\ \hline
				\multicolumn{6}{l}{Schaffer} \\ \hline 
				Scaling & Average & Average & Average func & Average \# & sum \\ 
				factor & CPU time  & Func val & -tion eval & generation & complexity \\ \hline 
				0 & 32.6328 & 7.6501 & 6.00 & 15 & 61.2829\\  
				0.1 & 32.6520 & 6.5712 & 6.00 & 15 & 60.2232\\ 
				0.2 & 34.2102 & 6.9237 & 6.00 & 15 & 62.1339\\  
				0.3 & 36.9380 & 5.0441 & 6.00 & 15 & 62.9822\\ 
				0.4 & 33.3884 & 5.5068 & 6.00 & 15 & 59.8952\\ 
				0.5 & 31.9806 & 6.0328 & 6.00 & 15 & 59.0134\\ 
				0.6 & 33.9578 & 7.1596 & 6.00 & 15 & 62.1174\\ 
				0.7 & 33.8586 & 8.2574 & 6.00 & 15 & 63.1161\\  
				0.8 & 39.3177 & 8.3107 & 6.00 & 15 & 68.6283\\  
				0.9 & 165.6111 & 9.1197 & 6.00 & 15 & 195.7308\\ \hline 
				\multicolumn{5}{l}{Minimum} & 59.0134\\ \hline
				
			\end{tabular}
		}
		
	\end{table}

	\begin{table}[hbt!]
		\centering
		\caption{Performance Summary for Each \(\kappa\) on the 2D Rastrigin Function for \(\kappa = 0, 0.1, 0.2, 0.3\)}
		The CPU time complexity, final function value reached, average number of function evaluations, and number of generations to reach stopping condition for each \(\kappa = 0,0.1,\dots,1\). Note that \(\kappa=0\) corresponds to having no step-size correction and \(\kappa=1\) corresponds to original PSA-CMA-ES.
		\label{table:RastriginEmpricalAllruns}
		\resizebox{0.99\textwidth}{!}{
			\fontsize{12}{15}\normalsize
			\begin{tabular}{l l l l l l l l l l l l l l l l l} \hline          
				\multicolumn{2}{l}{Rastrigin function } & \multicolumn{3}{l}{Interval = [1,5] for Rastrigin} & \multicolumn{12}{l}{} \\ \hline 
				Scale & \multicolumn{4}{l}{0 = step-size not} & \multicolumn{4}{l}{ } & \multicolumn{4}{l}{ } & \multicolumn{4}{l}{ } \\ 
				factor & \multicolumn{4}{l}{corrected (CSA only)} & \multicolumn{4}{l}{0.1000} & \multicolumn{4}{l}{0.2000} & \multicolumn{4}{l}{0.3000} \\ \hline 
				Run & CPU & Func & func & \# & CPU & Func & func & \# & CPU & Func & func & \# & CPU & Func & func & \# \\
				\# & time & val & eval & gens & time & val & eval & gens & time & val & eval & gens & time & val & eval & gens \\ \hline  
				1 & 31.1408 & 24.9301 & 6 & 15 & 24.4375 & 79.4132 & 102 & 15 & 36.3850 & 65.1416 & 66 & 15 & 38.1262 & 67.2522 & 206 & 15 \\ 
				2 & 31.5686 & 4.6190 & 6 & 15 & 24.7188 & 86.7120 & 181 & 15 & 30.7332 & 52.7722 & 175 & 15 & 31.5723 & 60.8930 & 131 & 15 \\  
				3 & 31.1153 & 28.5317 & 6 & 15 & 23.5781 & 54.9264 & 51 & 15 & 30.4812 & 64.8602 & 94 & 15 & 31.1168 & 2.0331 & 33 & 15 \\ 
				4 & 31.1699 & 46.2346 & 6 & 15 & 24.0469 & 35.1471 & 194 & 15 & 31.2968 & 55.3296 & 103 & 15 & 32.3385 & 40.1777 & 88 & 15 \\ 
				5 & 31.4700 & 33.7588 & 6 & 15 & 22.8438 & 18.0735 & 56 & 15 & 30.4620 & 87.9256 & 92 & 15 & 32.6149 & 63.3648 & 146 & 15 \\ 
				6 & 31.0051 & 13.1706 & 6 & 15 & 22.1563 & 45.3845 & 59 & 15 & 31.0727 & 52.7394 & 104 & 15 & 32.8661 & 60.6917 & 143 & 15 \\ 
				7 & 31.4098 & 27.4104 & 6 & 15 & 24.0625 & 30.3568 & 78 & 15 & 31.8713 & 69.1059 & 186 & 15 & 32.0672 & 58.0458 & 105 & 15 \\ 
				8 & 31.4929 & 22.3948 & 6 & 15 & 21.7500 & 62.0905 & 136 & 15 & 30.4063 & 13.1422 & 41 & 15 & 32.4641 & 34.2258 & 65 & 15 \\  
				9 & 30.9985 & 13.5859 & 6 & 15 & 21.9844 & 61.2134 & 103 & 15 & 31.3431 & 82.4765 & 95 & 15 & 31.7421 & 19.7440 & 57 & 15 \\ 
				10 & 31.9592 & 21.5929 & 6 & 15 & 23.5313 & 31.8685 & 52 & 15 & 30.2258 & 21.5497 & 30 & 15 & 32.0735 & 24.8739 & 30 & 15 \\  
				11 & 31.2760 & 11.9113 & 6 & 15 & 21.8281 & 64.1341 & 59 & 15 & 30.3058 & 74.2946 & 47 & 15 & 31.7912 & 21.2775 & 24 & 15 \\ 
				12 & 31.1790 & 29.5408 & 6 & 15 & 22.7344 & 65.7794 & 215 & 15 & 31.3800 & 39.1699 & 70 & 15 & 32.5512 & 60.6917 & 72 & 15 \\ 
				13 & 32.0814 & 9.4169 & 6 & 15 & 22.7344 & 115.4665 & 216 & 15 & 30.9000 & 40.2241 & 53 & 15 & 33.1298 & 54.4007 & 155 & 15 \\ 
				14 & 31.3309 & 6.9051 & 6 & 15 & 22.5156 & 51.9484 & 188 & 15 & 31.7809 & 53.3399 & 207 & 15 & 33.4846 & 104.7285 & 159 & 15 \\  
				15 & 34.7789 & 21.4414 & 6 & 15 & 21.7500 & 73.0362 & 135 & 15 & 32.6654 & 60.7838 & 180 & 15 & 35.8225 & 49.9147 & 35 & 15 \\ 
				16 & 32.0937 & 36.4256 & 6 & 15 & 22.6250 & 61.6516 & 195 & 15 & 33.5543 & 79.1985 & 99 & 15 & 32.9145 & 15.0931 & 29 & 15 \\  
				17 & 32.0933 & 9.6595 & 6 & 15 & 23.2969 & 46.0945 & 90 & 15 & 32.2482 & 24.8757 & 93 & 15 & 32.7092 & 61.7846 & 65 & 15 \\ 
				18 & 32.5546 & 30.6080 & 6 & 15 & 24.7344 & 72.0132 & 104 & 15 & 31.8681 & 63.8158 & 122 & 15 & 32.9955 & 2.0927 & 17 & 15 \\ 
				19 & 31.9799 & 13.8909 & 6 & 15 & 24.4063 & 71.7114 & 65 & 15 & 32.0063 & 25.4157 & 62 & 15 & 33.8483 & 99.4937 & 155 & 15 \\ 
				20 & 32.6620 & 15.6009 & 6 & 15 & 24.3594 & 24.9884 & 60 & 15 & 31.5091 & 34.4566 & 92 & 15 & 33.1615 & 4.9782 & 30 & 15 \\ \hline
				Average & 31.7680 & 21.0815 & 6 & 15 & 23.2047 & 57.6005 & 117 & 15 & 31.6248 & 53.0309 & 101 & 15 & 32.9695 & 45.2879 & 87 & 15 \\ \hline
				
			\end{tabular}
		}
		
	\end{table}
	
	\begin{table}[hbt!]
		\centering
		\caption{Performance Summary for Each \(\kappa\) on the 2D Rastrigin Function for \(\kappa = 0.4, 0.5, 0.6, 0.7\)}
		The CPU time complexity, final function value reached, average number of function evaluations, and number of generations to reach stopping condition for each \(\kappa = 0,0.1,\dots,1\). Note that \(\kappa=0\) corresponds to having no step-size correction and \(\kappa=1\) corresponds to original PSA-CMA-ES.
		\label{table:RastriginEmpricalAllruns_2}
		\resizebox{0.99\textwidth}{!}{
			\fontsize{12}{15}\normalsize
			\begin{tabular}{l l l l l l l l l l l l l l l l l} \hline          
				\multicolumn{2}{l}{Rastrigin function } & \multicolumn{3}{l}{Interval = [1,5] for Rastrigin} & \multicolumn{12}{l}{} \\ \hline 
				Scale & \multicolumn{4}{l}{ } & \multicolumn{4}{l}{ } & \multicolumn{4}{l}{ } & \multicolumn{4}{l}{ } \\ 
				factor & \multicolumn{4}{l}{0.4000} & \multicolumn{4}{l}{0.5000} & \multicolumn{4}{l}{0.6000} & \multicolumn{4}{l}{0.7000} \\ \hline 
				Run & CPU & Func & function & \# & CPU & Func & func & \# & CPU & Func & func & \# & CPU & Func & func & \# \\
				\# & time & val & eval & gens & time & val & eval & gens & time & val & eval & gens & time & val & eval & gens \\ \hline  
				1 & 30.6717 & 60.6917 & 120 & 15 & 33.9963 & 8.6438 & 30 & 15 & 32.7623 & 16.9150 & 13 & 15 & 31.8569 & 12.7187 & 12 & 15 \\  
				2 & 30.1432 & 23.9742 & 30 & 15 & 31.3888 & 1.9901 & 14 & 15 & 31.4014 & 5.3041 & 12 & 15 & 31.5262 & 25.1027 & 6 & 15 \\  
				3 & 30.3138 & 1.9994 & 14 & 15 & 32.2391 & 5.2063 & 9 & 15 & 32.1327 & 1.7178 & 14 & 15 & 31.2241 & 20.7890 & 6 & 15 \\  
				4 & 29.8384 & 3.9799 & 14 & 15 & 31.5227 & 5.1659 & 14 & 15 & 32.3463 & 51.7408 & 99 & 15 & 31.5300 & 2.3331 & 24 & 15 \\ 
				5 & 30.2371 & 12.9369 & 30 & 15 & 31.6634 & 26.1798 & 13 & 15 & 32.8732 & 1.0321 & 27 & 15 & 31.8046 & 27.2669 & 25 & 15 \\ 
				6 & 32.6920 & 3.9799 & 21 & 15 & 32.3713 & 19.8992 & 23 & 15 & 31.8101 & 5.0576 & 12 & 15 & 31.4754 & 8.7502 & 13 & 15 \\  
				7 & 30.1863 & 13.3722 & 27 & 15 & 31.9468 & 7.9597 & 21 & 15 & 32.8274 & 1.1223 & 14 & 15 & 31.9185 & 39.3866 & 6 & 15 \\  
				8 & 32.0090 & 44.7726 & 135 & 15 & 32.8206 & 13.1002 & 18 & 15 & 31.9782 & 3.9885 & 15 & 15 & 31.1621 & 19.5321 & 6 & 15 \\ 
				9 & 30.1279 & 8.9546 & 16 & 15 & 32.0132 & 0.9951 & 31 & 15 & 32.2197 & 0.9971 & 17 & 15 & 31.5664 & 15.4654 & 6 & 15 \\ 
				10 & 30.3988 & 5.5769 & 21 & 15 & 35.4773 & 0.9965 & 15 & 15 & 32.3597 & 3.9800 & 28 & 15 & 31.5865 & 15.9816 & 12 & 15 \\ 
				11 & 31.1976 & 5.3728 & 23 & 15 & 31.9473 & 19.9119 & 13 & 15 & 31.7013 & 8.9680 & 15 & 15 & 31.6663 & 16.6010 & 16 & 15 \\ 
				12 & 30.4420 & 0.9950 & 14 & 15 & 32.6856 & 15.9193 & 13 & 15 & 32.3069 & 4.9893 & 14 & 15 & 32.0320 & 11.0433 & 19 & 15 \\  
				13 & 30.7063 & 6.3442 & 24 & 15 & 31.8984 & 9.9496 & 22 & 15 & 31.4999 & 4.9762 & 9 & 15 & 31.0774 & 5.4500 & 6 & 15 \\ 
				14 & 31.3515 & 9.2690 & 17 & 15 & 32.5662 & 1.9903 & 17 & 15 & 32.2140 & 36.3328 & 6 & 15 & 31.5672 & 9.6635 & 8 & 15 \\ 
				15 & 30.9132 & 12.9345 & 54 & 15 & 25.8104 & 0.0058 & 12 & 12 & 31.8664 & 5.2175 & 8 & 15 & 31.2062 & 32.5926 & 6 & 15 \\  
				16 & 30.9023 & 46.2031 & 25 & 15 & 31.8790 & 4.9748 & 28 & 15 & 32.7932 & 8.9549 & 22 & 15 & 31.7098 & 33.8534 & 6 & 15 \\ 
				17 & 32.4904 & 24.8739 & 96 & 15 & 32.1926 & 12.9344 & 17 & 15 & 31.9237 & 15.9195 & 12 & 15 & 31.2807 & 31.9219 & 6 & 15 \\ 
				18 & 30.5094 & 0.9950 & 17 & 15 & 32.7991 & 6.1282 & 14 & 15 & 32.7866 & 28.4838 & 17 & 15 & 32.2711 & 29.9641 & 6 & 15 \\ 
				19 & 31.6274 & 8.9546 & 50 & 15 & 32.1427 & 8.9547 & 17 & 15 & 32.2087 & 2.0473 & 11 & 15 & 32.2881 & 46.7398 & 10 & 15 \\  
				20 & 31.7462 & 73.6257 & 148 & 15 & 32.8350 & 0.9950 & 25 & 15 & 33.0562 & 8.7385 & 7 & 15 & 32.4103 & 5.1851 & 14 & 15 \\ \hline
				Average & 30.9252 & 18.4903 & 45 & 15 & 32.1098 & 8.5950 & 18 & 15 & 32.2534 & 10.8242 & 19 & 15 & 31.6580 & 20.5170 & 11 & 15 \\ \hline
				
			\end{tabular}
		}
		
	\end{table}
	
	\begin{table}[hbt!]
		\centering
		\caption{Performance Summary for Each \(\kappa\) on the 2D Rastrigin Function for \(\kappa = 0.8, 0.9, 1.0\) }
		The CPU time complexity, final function value reached, average number of function evaluations, and number of generations to reach stopping condition for each \(\kappa = 0,0.1,\dots,1\). Note that \(\kappa=0\) corresponds to having no step-size correction and \(\kappa=1\) corresponds to original PSA-CMA-ES.
		\label{table:RastriginEmpricalAllruns_3}
		\resizebox{0.99\textwidth}{!}{
			\fontsize{12}{15}\normalsize
			\begin{tabular}{l l l l l l l l l l l l l l l l l} \hline          
				\multicolumn{2}{l}{Rastrigin function } & \multicolumn{3}{l}{Interval = [1,5] for Rastrigin} & \multicolumn{12}{l}{} \\ \hline 
				Scale & \multicolumn{4}{l}{ } & \multicolumn{4}{l}{ } & \multicolumn{4}{l}{1 = general} & \multicolumn{4}{l}{ } \\ 
				factor & \multicolumn{4}{l}{0.8000} & \multicolumn{4}{l}{0.9000} & \multicolumn{4}{l}{PSA-CMA-ES} & \multicolumn{4}{l}{ } \\ \hline 
				Run & CPU & Func & func & \# & CPU & Func & function & \# & CPU & Func & func & \# &   &   &   &  \\
				\# & time & val & eval & gens & time & val & eval & gens & time & val & eval & gens &   &   &   &   \\ \hline  
				1 & 32.0452 & 11.6659 & 6 & 15 & 32.8212 & 59.3563 & 6 & 15 & 119.4307 & 41.2789 & 6 & 15 & \multicolumn{4}{l}{} \\  
				2 & 32.2467 & 27.9563 & 6 & 15 & 32.9347 & 36.8837 & 6 & 15 & 31.1844 & 8.5050 & 6 & 15 & \multicolumn{4}{l}{} \\ 
				3 & 31.5090 & 47.8848 & 6 & 15 & 31.1912 & 37.6438 & 6 & 15 & 34.3967 & 35.2149 & 6 & 15 & \multicolumn{4}{l}{} \\ 
				4 & 32.2324 & 13.6714 & 6 & 15 & 32.1401 & 25.9330 & 6 & 15 & 42.4419 & 38.0120 & 6 & 15 & \multicolumn{4}{l}{} \\  
				5 & 32.1070 & 22.8479 & 6 & 15 & 32.9332 & 52.4305 & 6 & 15 & 46.9828 & 26.6687 & 6 & 15 & \multicolumn{4}{l}{} \\  
				6 & 32.4219 & 17.7171 & 6 & 15 & 32.2846 & 28.2692 & 6 & 15 & 31.9897 & 30.1167 & 6 & 15 & \multicolumn{4}{l}{} \\ 
				7 & 31.8999 & 10.0187 & 6 & 15 & 31.3483 & 38.7073 & 6 & 15 & 35.0829 & 46.6558 & 6 & 15 & \multicolumn{4}{l}{} \\  
				8 & 32.6407 & 35.5850 & 6 & 15 & 31.7662 & 32.2996 & 6 & 15 & 33.3601 & 42.1097 & 6 & 15 & \multicolumn{4}{l}{} \\ 
				9 & 32.6715 & 62.1696 & 6 & 15 & 32.6017 & 22.2138 & 6 & 15 & 34.2264 & 8.2475 & 6 & 15 & \multicolumn{4}{l}{} \\ 
				10 & 32.4909 & 27.4574 & 6 & 15 & 31.6150 & 40.5293 & 6 & 15 & 42.1668 & 11.6067 & 6 & 15 & \multicolumn{4}{l}{} \\ 
				11 & 32.3702 & 5.3721 & 6 & 15 & 33.0422 & 55.3996 & 6 & 15 & 31.5921 & 32.3726 & 6 & 15 & \multicolumn{4}{l}{} \\  
				12 & 32.2746 & 36.7432 & 6 & 15 & 31.7015 & 50.3359 & 6 & 15 & 33.2432 & 25.0511 & 6 & 15 & \multicolumn{4}{l}{} \\ 
				13 & 32.6757 & 21.9755 & 6 & 15 & 32.3399 & 54.2043 & 6 & 15 & 39.0653 & 25.8847 & 6 & 15 & \multicolumn{4}{l}{} \\  
				14 & 31.9320 & 35.6981 & 6 & 15 & 31.7793 & 27.1840 & 6 & 15 & 90.4119 & 46.5032 & 6 & 15 & \multicolumn{4}{l}{} \\  
				15 & 32.6726 & 24.2976 & 6 & 15 & 33.4121 & 35.1416 & 6 & 15 & 46.4903 & 47.7954 & 6 & 15 & \multicolumn{4}{l}{} \\  
				16 & 32.4198 & 27.3697 & 6 & 15 & 32.0373 & 16.0352 & 6 & 15 & 39.2397 & 26.1588 & 6 & 15 & \multicolumn{4}{l}{} \\ 
				17 & 33.0279 & 20.3560 & 6 & 15 & 32.6411 & 19.4591 & 6 & 15 & 32.6473 & 20.6702 & 6 & 15 & \multicolumn{4}{l}{} \\  
				18 & 32.3373 & 32.5830 & 6 & 15 & 32.1138 & 34.9183 & 6 & 15 & 79.2915 & 9.7315 & 6 & 15 & \multicolumn{4}{l}{} \\ 
				19 & 33.6175 & 27.0031 & 6 & 15 & 33.2098 & 41.4315 & 6 & 15 & 37.2033 & 14.3622 & 6 & 15 & \multicolumn{4}{l}{} \\ 
				20 & 32.9852 & 2.5566 & 6 & 15 & 33.0280 & 47.7506 & 6 & 15 & 39.2854 & 19.8041 & 6 & 15 & \multicolumn{4}{l}{} \\ \hline
				Average & 32.4289 & 25.5464 & 6 & 15 & 32.3471 & 37.8063 & 6 & 15 & 45.9866 & 27.8375 & 6 & 15 & \multicolumn{4}{l}{} \\ \hline
				
			\end{tabular}
		}
		
	\end{table}

	\begin{table}[hbt!]
		\centering
		\caption{Performance Summary for Each \(\kappa\) on the 2D Schaffer Function}
		The CPU time complexity, final function value reached, average number of function evaluations, and number of generations to reach stopping condition for each \(\kappa = 0,0.1,\dots,1\). Note that \(\kappa=0\) corresponds to having no step-size correction and \(\kappa=1\) which corresponds to the original PSA-CMA-ES algorithm was omitted as the complexity was too high to run 20 separate runs of 15 generations each.
		\label{table:SchafferEmpricalAllruns}
		\resizebox{0.99\textwidth}{!}{
			\begin{tabular}{|l |l |l |l |l |l |l |l |l |l |l |l |l |l |l |l |l |l |l |l |l|} \hline   
				\multicolumn{2}{|l|}{Schaffer function } & \multicolumn{3}{|l|}{Interval = [10,100] for Schaffer} & \multicolumn{12}{|l|}{} &  &  &  &  \\ \hline 
				Scale factor & \multicolumn{4}{|l|}{0 = step-size note corrected (CSA only)} & \multicolumn{4}{|l|}{0.1000} & \multicolumn{4}{|l|}{0.2000} & \multicolumn{4}{|l|}{0.3000} & \multicolumn{4}{|l|}{0.4000} \\ \hline 
				Run \# & CPU time & Func val & function eval & \# generations & CPU time & Func val & function eval & \# generations & CPU time & Func val & function eval & \# generations & CPU time & Func val & function eval & \# generations & CPU time & Func val & function eval & \# generations \\ \hline 
				1 & 32.2807 & 7.0650 & 6 & 15 & 33.6482 & 0.5149 & 6 & 15 & 34.1626 & 6.6240 & 6 & 15 & 31.7344 & 7.1707 & 6 & 15 & 29.2075 & 9.3987 & 6 & 15 \\ 
				2 & 33.1100 & 11.6047 & 6 & 15 & 31.6488 & 7.4201 & 6 & 15 & 33.5734 & 9.6315 & 6 & 15 & 35.9258 & 7.6940 & 6 & 15 & 32.5787 & 4.8773 & 6 & 15 \\  
				3 & 32.0568 & 3.7639 & 6 & 15 & 34.4805 & 9.9426 & 6 & 15 & 31.3102 & 7.6940 & 6 & 15 & 35.4653 & 7.1707 & 6 & 15 & 31.7190 & 7.1707 & 6 & 15 \\ 
				4 & 31.3655 & 2.1145 & 6 & 15 & 32.8757 & 3.4530 & 6 & 15 & 34.0326 & 9.3987 & 6 & 15 & 36.3092 & 3.7583 & 6 & 15 & 32.4797 & 2.7702 & 6 & 15 \\ 
				5 & 32.5847 & 5.9498 & 6 & 15 & 33.2571 & 10.7824 & 6 & 15 & 33.4465 & 3.4248 & 6 & 15 & 36.0007 & 2.7916 & 6 & 15 & 31.7610 & 8.2396 & 6 & 15 \\  
				6 & 32.1138 & 7.7872 & 6 & 15 & 32.7152 & 1.3596 & 6 & 15 & 34.6938 & 6.6694 & 6 & 15 & 36.0619 & 1.7695 & 6 & 15 & 32.1004 & 2.7917 & 6 & 15 \\ 
				7 & 32.8578 & 3.3524 & 6 & 15 & 32.6844 & 4.9518 & 6 & 15 & 33.2845 & 6.3810 & 6 & 15 & 35.9434 & 11.2242 & 6 & 15 & 31.2818 & 4.8773 & 6 & 15 \\ 
				8 & 32.6864 & 12.3528 & 6 & 15 & 28.7471 & 10.0526 & 6 & 15 & 33.7735 & 14.7758 & 6 & 15 & 36.6620 & 4.8773 & 6 & 15 & 32.9181 & 6.1898 & 6 & 15 \\ 
				9 & 32.5141 & 11.9919 & 6 & 15 & 33.1766 & 2.8490 & 6 & 15 & 33.3470 & 6.6694 & 6 & 15 & 36.0377 & 5.7314 & 6 & 15 & 31.7361 & 5.2940 & 6 & 15 \\ 
				10 & 32.5549 & 9.8955 & 6 & 15 & 32.2795 & 6.6719 & 6 & 15 & 35.2015 & 3.9672 & 6 & 15 & 36.1955 & 3.4102 & 6 & 15 & 32.0594 & 7.2699 & 6 & 15 \\ 
				11 & 33.0904 & 6.2006 & 6 & 15 & 33.5394 & 3.1401 & 6 & 15 & 34.2338 & 9.3987 & 6 & 15 & 37.2394 & 8.2396 & 6 & 15 & 32.7297 & 5.7314 & 6 & 15 \\  
				12 & 31.7329 & 12.5422 & 6 & 15 & 33.0714 & 5.2940 & 6 & 15 & 33.7632 & 11.4552 & 6 & 15 & 37.2133 & 4.4809 & 6 & 15 & 37.1398 & 7.3124 & 6 & 15 \\ 
				13 & 33.0252 & 12.7022 & 6 & 15 & 33.3598 & 8.2265 & 6 & 15 & 34.5368 & 4.3989 & 6 & 15 & 37.6766 & 4.4810 & 6 & 15 & 33.9817 & 5.7314 & 6 & 15 \\  
				14 & 32.6619 & 3.2518 & 6 & 15 & 30.4516 & 8.2396 & 6 & 15 & 35.0162 & 5.5427 & 6 & 15 & 37.0796 & 4.1236 & 6 & 15 & 34.1537 & 7.1721 & 6 & 15 \\ 
				15 & 33.4630 & 3.7140 & 6 & 15 & 33.4092 & 6.0727 & 6 & 15 & 34.8667 & 4.1377 & 6 & 15 & 39.7107 & 1.5566 & 6 & 15 & 33.5571 & 4.4809 & 6 & 15 \\ 
				16 & 32.2826 & 3.7477 & 6 & 15 & 29.4900 & 11.0043 & 6 & 15 & 35.2015 & 0.6024 & 6 & 15 & 38.9421 & 6.6694 & 6 & 15 & 33.5841 & 6.1897 & 6 & 15 \\ 
				17 & 32.3343 & 7.4677 & 6 & 15 & 33.8899 & 9.0708 & 6 & 15 & 34.5991 & 6.1897 & 6 & 15 & 38.8162 & 2.6414 & 6 & 15 & 37.7586 & 6.1897 & 6 & 15 \\ 
				18 & 32.6292 & 13.6017 & 6 & 15 & 33.8642 & 8.8077 & 6 & 15 & 34.6260 & 7.1707 & 6 & 15 & 38.4099 & 3.0916 & 6 & 15 & 37.8768 & 2.7917 & 6 & 15 \\ 
				19 & 32.3006 & 9.7850 & 6 & 15 & 33.7445 & 4.6043 & 6 & 15 & 35.5285 & 7.1707 & 6 & 15 & 39.2581 & 2.8300 & 6 & 15 & 34.5427 & 2.2478 & 6 & 15 \\ 
				20 & 35.0111 & 4.1124 & 6 & 15 & 32.7070 & 8.9664 & 6 & 15 & 35.0065 & 7.1707 & 6 & 15 & 38.0790 & 7.1707 & 6 & 15 & 34.6026 & 3.4103 & 6 & 15 \\ 
				Average & 32.6328 & 7.6501 & 6 & 15 & 32.6520 & 6.5712 & 6 & 15 & 34.2102 & 6.9237 & 6 & 15 & 36.9381 & 5.0441 & 6 & 15 & 33.3884 & 5.5068 & 6 & 15 \\ \hline
				Scale factor & \multicolumn{4}{|l|}{0.5000} & \multicolumn{4}{|l|}{0.6000} & \multicolumn{4}{|l|}{0.7000} & \multicolumn{4}{|l|}{0.8000} & \multicolumn{4}{|l|}{0.9000} \\ \hline 
				Run \# & CPU time & Func val & function eval & \# generations & CPU time & Func val & function eval & \# generations & CPU time & Func val & function eval & \# generations & CPU time & Func val & function eval & \# generations & CPU time & Func val & function eval & \# generations \\ \hline 
				1 & 34.9447 & 4.4869 & 6 & 15 & 38.6262 & 5.2946 & 6 & 15 & 32.2997 & 6.3025 & 6 & 15 & 37.9748 & 6.6703 & 6 & 15 & 34.8891 & 4.4810 & 6.0000 & 15.0000 \\  
				2 & 31.2726 & 7.6940 & 6 & 15 & 32.1064 & 8.2590 & 6 & 15 & 32.7741 & 8.2396 & 6 & 15 & 32.4694 & 4.7551 & 6 & 15 & 53.1390 & 6.1307 & 6.0000 & 15.0000 \\  
				3 & 31.5944 & 11.3362 & 6 & 15 & 32.7252 & 11.4842 & 6 & 15 & 32.7411 & 12.5560 & 6 & 15 & 33.3774 & 8.3451 & 6 & 15 & 41.3691 & 5.4011 & 6.0000 & 15.0000 \\ 
				4 & 30.2471 & 4.7713 & 6 & 15 & 32.8183 & 2.6219 & 6 & 15 & 34.1528 & 4.3814 & 6 & 15 & 32.4914 & 11.3835 & 6 & 15 & 39.4772 & 18.5032 & 6.0000 & 15.0000 \\  
				5 & 31.4550 & 4.1074 & 6 & 15 & 36.7010 & 12.2125 & 6 & 15 & 33.2970 & 7.6618 & 6 & 15 & 34.8323 & 5.1894 & 6 & 15 & 44.8450 & 10.5240 & 6 & 15 \\ 
				6 & 31.2220 & 7.7155 & 6 & 15 & 32.8340 & 6.3335 & 6 & 15 & 33.3807 & 4.7501 & 6 & 15 & 34.1978 & 5.0417 & 6 & 15 & 64.1702 & 11.0048 & 6 & 15 \\ 
				7 & 31.7646 & 6.2889 & 6 & 15 & 34.0693 & 7.1223 & 6 & 15 & 32.6883 & 9.7808 & 6 & 15 & 77.4893 & 14.3874 & 6 & 15 & 35.8874 & 6.0571 & 6 & 15 \\ 
				8 & 31.8197 & 7.2535 & 6 & 15 & 32.8291 & 12.5618 & 6 & 15 & 37.8924 & 10.9651 & 6 & 15 & 33.8628 & 6.7552 & 6 & 15 & 72.1146 & 9.0246 & 6 & 15 \\ 
				9 & 31.5695 & 6.2002 & 6 & 15 & 33.4142 & 6.5425 & 6 & 15 & 33.4654 & 3.2977 & 6 & 15 & 33.8365 & 3.1543 & 6 & 15 & 46.1425 & 12.3920 & 6 & 15 \\  
				10 & 31.4269 & 4.1082 & 6 & 15 & 33.8178 & 8.2623 & 6 & 15 & 34.7398 & 9.5396 & 6 & 15 & 36.9573 & 7.6842 & 6 & 15 & 179.9999 & 8.1672 & 6 & 15 \\ 
				11 & 31.3548 & 5.7314 & 6 & 15 & 33.0098 & 6.0485 & 6 & 15 & 32.9613 & 3.4174 & 6 & 15 & 47.3943 & 11.4556 & 6 & 15 & 277.9109 & 10.3458 & 6 & 15 \\ 
				12 & 31.4597 & 5.2941 & 6 & 15 & 33.4093 & 5.7391 & 6 & 15 & 33.0483 & 8.5131 & 6 & 15 & 46.9266 & 19.2053 & 6 & 15 & 246.8716 & 7.8789 & 6 & 15 \\ 
				13 & 30.6965 & 4.8773 & 6 & 15 & 33.4769 & 7.1725 & 6 & 15 & 34.1169 & 7.0834 & 6 & 15 & 50.6388 & 6.9330 & 6 & 15 & 79.7734 & 15.2705 & 6 & 15 \\  
				14 & 32.0590 & 2.2483 & 6 & 15 & 33.6285 & 8.5776 & 6 & 15 & 33.9572 & 13.4987 & 6 & 15 & 33.9541 & 9.6439 & 6 & 15 & 377.8324 & 10.8914 & 6 & 15 \\  
				15 & 31.9169 & 7.4091 & 6 & 15 & 33.5778 & 4.5418 & 6 & 15 & 33.6633 & 12.3486 & 6 & 15 & 34.2130 & 5.0283 & 6 & 15 & 301.5669 & 9.2141 & 6 & 15 \\  
				16 & 33.0470 & 7.1709 & 6 & 15 & 33.3678 & 6.5431 & 6 & 15 & 34.0189 & 5.0566 & 6 & 15 & 35.1120 & 10.0984 & 6 & 15 & 58.1316 & 9.3283 & 6 & 15 \\  
				17 & 32.3445 & 6.6707 & 6 & 15 & 33.0487 & 3.5677 & 6 & 15 & 34.7311 & 11.0254 & 6 & 15 & 45.1374 & 9.0836 & 6 & 15 & 876.1363 & 5.6015 & 6 & 15 \\ 
				18 & 31.8136 & 5.7315 & 6 & 15 & 33.8254 & 6.5322 & 6 & 15 & 35.4761 & 5.2658 & 6 & 15 & 34.6863 & 8.5563 & 6 & 15 & 108.0307 & 9.8046 & 6 & 15 \\  
				19 & 34.9063 & 4.1063 & 6 & 15 & 36.8559 & 10.0439 & 6 & 15 & 34.2570 & 11.9161 & 6 & 15 & 34.1639 & 6.0756 & 6 & 15 & 320.7381 & 5.9656 & 6 & 15 \\  
				20 & 32.6974 & 7.4546 & 6 & 15 & 35.0142 & 3.7303 & 6 & 15 & 33.5116 & 9.5495 & 6 & 15 & 36.6378 & 6.7671 & 6 & 15 & 53.1966 & 6.4078 & 6 & 15 \\  
				Average & 31.9806 & 6.0328 & 6 & 15 & 33.9578 & 7.1596 & 6 & 15 & 33.8587 & 8.2575 & 6 & 15 & 39.3177 & 8.3107 & 6 & 15 & 165.6111 & 9.1197 & 6 & 15 \\ \hline
				
			\end{tabular}
		}
		
	\end{table}

\end{document}